\documentclass{article}
\pdfoutput=1
\newcommand{\fullpaper}{}
\newcommand{\conferencepaper}[1]{}

\usepackage{colt10e}

\usepackage{amsthm,amsfonts,amsmath,amssymb,epsfig,color,float,graphicx,verbatim}

\newtheorem{theorem}{Theorem}

\newtheorem{lemma}{Lemma}

\newtheorem{example}{Example}

\renewcommand{\Pr}{\mathbb{P}}
\newcommand{\reals}{\mathbb{R}}
\newcommand{\E}{\mathbb{E}}

\newcommand{\argmin}[1]{\underset{#1}{\mathrm{argmin}}}

\newcommand{\summ}{\displaystyle \sum}

\newcommand{\be}{\mathbf{e}}
\newcommand{\bx}{\mathbf{x}}
\newcommand{\bw}{\mathbf{w}}

\newcommand{\bu}{\mathbf{u}}

\newcommand{\bomega}{\boldsymbol{\omega}}

\newcommand{\Dcal}{\mathcal{D}}

\newcommand{\Wcal}{\mathcal{W}}

\newcommand{\norm}[1]{\|#1\|}
\newcommand{\inner}[1]{\langle#1\rangle}

\newcommand{\secref}[1]{Sec.~\ref{#1}}
\newcommand{\subsecref}[1]{Subsection~\ref{#1}}
\newcommand{\figref}[1]{Fig.~\ref{#1}}
\renewcommand{\eqref}[1]{Eq.~(\ref{#1})}
\newcommand{\lemref}[1]{Lemma~\ref{#1}}

\newcommand{\thmref}[1]{Thm.~\ref{#1}}

\newcommand{\appref}[1]{Appendix~\ref{#1}}
\newcommand{\algoref}[1]{Algorithm~\ref{#1}}
\newcommand{\subref}[1]{Subroutine~\ref{#1}}

\newenvironment{algorithm}[1][\  ] %
{ \rm
\begin{tabbing}
....\=.....\=.....\=.....\=.....\=  \+ \kill
} %
{\end{tabbing} }

\floatstyle{ruled}
\newfloat{Algorithm}{H}{alg}
\newfloat{Subroutine}{H}{sub}

\newenvironment{Balgorithm} %
{
\begin{minipage}{1.0\linewidth} \begin{algorithm} %
} { \end{algorithm} \end{minipage} }

\title{
Online Learning of Noisy Data with Kernels
}
\author{Nicol\`o Cesa-Bianchi\\
Universit\`a degli Studi di Milano\\
\texttt{cesa-bianchi@dsi.unimi.it}
\And
Shai Shalev Shwartz\\
The Hebrew University\\
\texttt{shais@cs.huji.ac.il}
\And
Ohad Shamir\\
The Hebrew University\\
\texttt{ohadsh@cs.huji.ac.il}}

\begin{document}

\maketitle

\begin{abstract}
  We study online learning when individual instances are corrupted by
  adversarially chosen random noise. We assume the noise distribution is
  unknown, and may change over time with no restriction other than having zero
  mean and bounded variance.  Our technique relies on a family of unbiased
  estimators for non-linear functions, which may be of independent interest.
  We show that a variant of online gradient descent can learn functions in any
  dot-product (e.g., polynomial) or Gaussian kernel space with any analytic
  convex loss function.  Our variant uses randomized estimates that need to
  query a random number of noisy copies of each instance, where with high
  probability this number is upper bounded by a constant. Allowing such
  multiple queries cannot be avoided: Indeed, we show that online learning is
  in general impossible when only one noisy copy of each instance can be
  accessed.
\end{abstract}

%%%%%%%%%%%%%%%%%%%%%%%%%%%%%%%%%%%%%%%%%%%%%%%%%%
%%%%%%%%%%%%%%%%%%%%%%%%%%%%%%%%%%%%%%%%%%%%%%%%%%
%%%%%%%%%%%%%%%%%%%%%%%%%%%%%%%%%%%%%%%%%%%%%%%%%%
\section{Introduction}
%
% Foundational paper on robustness of online learning to noise, including kernels.
% 
% Possible motivations: Communication constraints, measurement errors,
% measurement costs...
% 
% The techniques of this paper allow a \emph{continuous} tradeoff between the
% sample size and the amount of information that needs to be gathered on each
% training example.
%
In many machine learning applications training data are typically collected by
measuring certain physical quantities. Examples include bioinformatics, medical
tests, robotics, and remote sensing. These measurements have errors that may be
due to several reasons: sensor costs, communication constraints, or intrinsic
physical limitations.  In all such cases, the learner trains on a distorted
version of the actual ``target'' data, which is where the learner's predictive
ability is eventually evaluated.  In this work we investigate the extent to
which a learning algorithm can achieve a good predictive performance when
training data are corrupted by noise with unknown distribution.

We prove upper and lower bounds on the learner's cumulative loss in the
framework of online learning, where examples are generated by an arbitrary and
possibily adversarial source. We model the measurement error via a random
perturbation which affects each instance observed by the learner. We do not
assume any specific property of the noise distribution other than zero-mean and
bounded variance. Moreover, we allow the noise distribution to change at every
step in an adversarial way and fully hidden from the learner. Our positive
results are quite general: by using a randomized unbiased estimate for the loss
gradient and a randomized feature mapping to estimate kernel values, we show
that a variant of online gradient descent can learn functions in any
dot-product (e.g., polynomial) or Gaussian RKHS under any given analytic convex
loss function. Our techniques are readily extendable to other kernel types as
well.

In order to obtain unbiased estimates of loss gradients and kernel
values, we allow the learner to query a random number of independently
perturbed copies of the current unseen instance. We show how low-variance
estimates can be computed using a number of queries that is \textsl{constant}
with high probability.
This is in sharp contrast with standard averaging techniques which attempts
to directly estimate the noisy instance, as these require a sample whose size
depends on the scale of the problem.
%NCB Not sure if this is to be stressed: we do not now how to do the trade-off
%    and we may fall back to the case when p depends on the variance of ||x||.
% Our query distribution, instead, depends on a single tunable parameter,
% which the the learner may use to trade off, in a controlled way,
% the expected number of queries with the corresponding predictive performance.
Finally, we formally show that learning is impossible, even without kernels,
when only one perturbed copy of each instance can be accessed. This is true
for essentially any reasonable loss function.

Our paper is organized as follows. In the next subsection we discuss related
work. In \secref{sec:setting} we introduce our setting and justify some of our
choices. In \secref{sec:results} we present our main results but before that,
in \secref{sec:techniques}, we discuss the techniques used to obtain them. In
the same section, we also explain why existing techniques are insufficient to
deal with our problem. The detailed proofs and subroutine implementations
appear in \secref{sec:proofs}, with some of the more technical lemmas and
proofs relegated to \fullpaper{the
  appendix}\conferencepaper{\cite{CesShalSham10}}. We wrap up with a discussion
on possible avenues for future work in \secref{sec:futurework}.

\subsection{Related Work}
In the machine learning literature, the problem of learning from noisy
examples, and, in particular, from noisy training instances, has traditionally
received a lot of attention ---see, for example, the recent
survey~\cite{NOF10}.  On the other hand, there are comparably few
theoretically-principled studies on this topic. Two of them focus on models
quite different from the one studied here: random attribute noise in PAC
boolean learning~\cite{BJT03,GS95}, and malicious noise~\cite{KL93,CDFSS99}.
In the first case, learning is restricted to classes of boolean functions and
the noise must be independent across each boolean coordinate.  In the second
case, an adversary is allowed to perturb a small fraction of the training
examples in an arbitrary way, making learning impossible in a strong
informational sense unless this perturbed fraction is very small (of the order
of the desired accuracy for the predictor).

The previous work perhaps closest to the one presented here is~\cite{Lit91},
where binary classification mistake bounds are proven for the online Winnow
algorithm in the presence of attribute errors. Similarly to our setting,
the sequence of instances observed by the learner is chosen by an adversary.
However, in~\cite{Lit91} the noise is generated by an adversary, who may
change the value of each attribute in an arbitrary way. The final mistake
bound, which only applies when the noiseless data sequence is linearly
separable without kernels, depends on the sum of all adversarial
perturbations.

%%%%%%%%%%%%%%%%%%%%%%%%%%%%%%%%%%%%%%%%%%%%%%%%%%
%%%%%%%%%%%%%%%%%%%%%%%%%%%%%%%%%%%%%%%%%%%%%%%%%%
%%%%%%%%%%%%%%%%%%%%%%%%%%%%%%%%%%%%%%%%%%%%%%%%%%
\section{Setting}\label{sec:setting}
We consider a setting where the goal is to predict
values $y\in\reals$ based on instances $\bx\in\reals^d$.
In this paper we focus on kernel-based linear predictors of
the form $\bx\mapsto\inner{\bw,\Psi(\bx)}$, where $\Psi$ is a feature mapping
into some reproducing kernel Hilbert space (RKHS). We assume there exists a
kernel function that efficiently implements dot products in that space,
i.e., $k(\bx,\bx')=\inner{\Psi(\bx),\Psi(\bx')}$. Note that a special case of
this setting is linear kernels, where $\Psi(\cdot)$ is the
identity mapping and $k(\bx,\bx')=\inner{\bx,\bx'}$.

The standard online learning protocol for linear prediction with kernels
is defined as follows:
at each round $t$, the learner picks a linear hypothesis $\bw_t$ from the
RKHS. The adversary then picks an example $(\bx_t,y_t)$ and reveals it
to the learner.
The loss suffered by the learner is $\ell(\inner{\bw_t,\Psi(\bx_t)},y_t)$,
where $\ell$ is a known and fixed loss function.
The goal of the learner is to minimize \emph{regret} with respect to a
fixed convex set of hypotheses $\Wcal$, namely
\[
\sum_{t=1}^{T}\ell(\inner{\bw_t,\Psi(\bx_t)},y_t)-\min_{\bw\in
  \Wcal}\sum_{t=1}^{T}\ell(\inner{\bw,\Psi(\bx_t)},y_t).
\]
Typically, we wish to find a strategy for the learner, such that no matter what
is the adversary's strategy of choosing the sequence of examples, the
expression above is sub-linear in $T$.

We now make the following twist, which limits the information available to the
learner: instead of receiving $(\bx_t,y_t)$, the learner observes $y_t$ and is
given access to an \emph{oracle} $A_t$. On each call, $A_t$ returns an
independent copy of $\bx_t+Z_t$, where $Z_t$ is a zero-mean random vector with
some known finite bound on its variance (in the sense that
$\E\bigl[\norm{{Z}_t}^2\bigr] \leq a$ for some uniform constant $a$).  In
general, the distribution of $Z_t$ is unknown to the learner. It might be
chosen by the adversary, and change from round to round or even between
consecutive calls to $A_t$. Note that here we assume that $y_t$ remains
unperturbed, but we emphasize that this is just for simplicity - our techniques
can be readily extended to deal with noisy values as well.

The learner may call $A_t$ more than once. In fact, as we discuss later on,
being able to call $A_t$ more than once is necessary for the learner to have
any hope to succeed. On the other hand, if the learner calls $A_t$ an unlimited
number of times, it can reconstruct $\bx_t$ arbitrarily well by averaging, and
we are back to the standard learning setting.
% However, the motivation of our
% paper are those cases where calling $A_t$ many times is prohibitively expensive
% or impossible.
In this paper we focus on learning algorithms that call
$A_t$ only a small, essentially constant number of times, which depends only on
our choice of loss function and kernel (rather than $T$, the norm of $\bx_t$,
or the variance of $Z_t$, which will happen with na\"{i}ve averaging techniques).
Moreover, since the number of queries is bounded with very high 
probability, one can even produce an algorithm with an absolute bound on the 
number of queries, which will fail or introduce some bias with an arbitrarily 
small probability. For simplicity, we ignore these issues in this paper.

In this setting, we wish to minimize the regret in hindsight with respect to
the unperturbed data and averaged over the noise introduced by the oracle,
namely
\begin{equation}\label{eq:regret_alt2}
\E\left[\sum_{t=1}^{T}\ell(\inner{\bw_t,\Psi(\bx_t)},y_t)
-
    \min_{\bw\in \Wcal}\sum_{t=1}^{T}\ell(\inner{\bw,\Psi(\bx_t)},y_t)\right]
\end{equation}
where the random quantities are the predictors $\bw_1,\bw_2,\dots$ generated
by the learner, which depend on the observed noisy instances
(in \fullpaper{the appendix}\conferencepaper{\cite{CesShalSham10}}, we briefly discuss alternative regret measures, and why they
are unsatisfactory). This kind of regret is relevant where we actually wish to
learn from data, without the noise causing a hindrance. In particular, consider
the batch setting, where the examples $\{(\bx_t,y_t)\}_{t=1}^{T}$ are actually
sampled i.i.d.\ from some unknown distribution, and we wish to find a predictor
which minimizes the expected loss $\E[\ell(\inner{\bw,\bx},y)]$ with respect to new
examples $(\bx,y)$. Using standard online-to-batch conversion techniques, if we
can find an online algorithm with a sublinear bound on~\eqref{eq:regret_alt2},
then it is possible to construct learning algorithms for the batch setting
which are robust to noise. That is, algorithms generating a predictor $\bw$
with close to minimal expected loss $\E[\ell(\inner{\bw,\bx},y)]$ among all
$\bw\in \Wcal$. 

While our techniques are quite general, the exact algorithmic and theoretical
results depend a lot on which loss function and kernel is used. Discussing the
loss function first, we will assume that $\ell(\inner{\bw,\Psi(\bx)},y)$ is a
convex function of $\bw$ for each example $(\bx,y)$. Somewhat abusing
notation, we assume the loss can be written either as
$
\ell(\inner{\bw,\Psi(\bx)},y) = f(y\inner{\bw,\Psi(\bx)})
$
or as
$
\ell(\inner{\bw,\Psi(\bx)},y) = f(\inner{\bw,\Psi(\bx)}-y)
$
for some function $f$.
We refer to the first type as \emph{classification losses},
as it encompasses most reasonable losses for classification, where
$y\in \{-1,+1\}$ and the goal is to predict the label. We refer to the second
type as \emph{regression losses}, as it encompasses most reasonable regression
losses, where $y$ takes arbitrary real values. For simplicity, we present some
of our results in terms of classification losses, but they all hold for
regression losses as well with slight modifications.

We present our results under the assumption that the loss function is
``smooth'', in the sense that $\ell'(a)$ can be written as
$\sum_{n=0}^{\infty}\gamma_n a^n$, for any $a$ in its domain. This assumption
holds for instance for the squared loss $\ell(a)=a^2$, the exponential loss
$\ell(a)=\exp(a)$, and smoothed versions of loss functions such as the hinge
loss and the absolute loss (we discuss examples in more details in
\subsecref{subsec:examples}). This assumption can be relaxed under certain
conditions, and this is further discussed in \subsecref{subsec:indian}.

Turning to the issue of kernels, we note that the general presentation of our
approach is somewhat hampered by the fact that it needs to be tailored to the
kernel we use. In this paper, we focus on two families of kernels:

\smallskip\noindent
\emph{Dot Product Kernels}: the kernel $k(\bx,\bx')$ can be written as a
function of $\inner{\bx,\bx'}$. Examples of such kernels $k(\bx,\bx')$ are
linear kernels $\inner{\bx,\bx'}$; homogeneous polynomial kernels
$(\inner{\bx,\bx'})^n$, inhomogeneous polynomial kernels
$(1+\inner{\bx,\bx'})^n$; exponential kernels $e^{\inner{\bx,\bx'}}$;
binomial kernels $(1+\inner{\bx,\bx'})^{-\alpha}$, and more
(see for instance~\cite{ScholkopfSm02,SteinwartChri08}).

\smallskip\noindent
\emph{Gaussian Kernels}: $k(\bx,\bx')=e^{-\norm{\bx-\bx'}^2/\sigma^2}$
for some $\sigma^2>0$.

\smallskip
Again, we emphasize that our techniques are extendable to other kernel types 
as well.

%%%%%%%%%%%%%%%%%%%%%%%%%%%%%%%%%%%%%%%%%%%%%%%%%%
%%%%%%%%%%%%%%%%%%%%%%%%%%%%%%%%%%%%%%%%%%%%%%%%%%
%%%%%%%%%%%%%%%%%%%%%%%%%%%%%%%%%%%%%%%%%%%%%%%%%%
\section{Techniques}\label{sec:techniques}
Our results are based on two key ideas: the use of online gradient descent
algorithms, and construction of unbiased gradient estimators in the kernel
setting. The latter is based on a general method to build unbiased estimators
for non-linear functions, which may be of independent interest.

%%%%%%%%%%%%%%%%%%%%%%%%%%%%%%%%%%%%%%%%%%%%%%%%%%
\subsection{Online Gradient Descent}
There exist well developed theory and algorithms for dealing with the standard
online learning setting, where the example $(\bx_t,y_t)$ is revealed after each
round, and for general convex loss functions. One of the simplest and most well
known ones is the online gradient descent algorithm due to
Zinkevich~\cite{Zinkevich03}.
Since this algorithm forms a basis for our algorithm in
the new setting, we briefly review it below (as adapted to our setting).

The algorithm initializes the classifier $\bw_1=0$. At round $t$, the algorithm
predicts according to $\bw_t$, and updates the learning rule according to
$\bw_{t+1}=P\bigl(\bw_t-\eta_t\nabla_t\bigr)$,
where $\eta_t$ is a suitably chosen constant which might depend on $t$;
$\nabla_t = \ell'\bigl(y_t\inner{\bw_t,\Psi(\bx_t)}\bigr)y_t\Psi(\bx_t)$ is the
\textsl{gradient} of $\ell\bigl(y_t\inner{\bw,\Psi(\bx_t)}\bigr)$
with respect to $\bw_t$; and $P$ is a projection operator on the
convex set $\Wcal$, on whose elements we wish to
achieve low regret. In particular, if we wish to compete with hypotheses of
bounded squared norm $B_{\bw}$, $P$ simply involves rescaling the norm of the
predictor so as to have squared norm at most $B_{\bw}$. With this algorithm,
one can prove regret bounds with respect to any $\bw\in \Wcal$.

A ``folklore'' result about this algorithm is that in fact, we do not need to 
update the predictor by the gradient at each step. Instead, it is enough to 
update by some random vector of bounded variance, which merely equals the 
gradient in expectation. This is a useful property in settings where 
$(\bx_t,y_t)$ is not revealed to the learner, and has been used before, such 
as in the online bandit setting (see for 
instance~\cite{CesaBianchiLu06,FlaxmanKalMcm05,AbernethyHazRakh08}). Here, we 
will use this property in a new way, in order to devise algorithms which are 
robust to noise. When the kernel and loss function are linear (e.g., 
$\Psi(\bx)=\bx$ and $\ell(a)=ca+b$ for some constants $b,c$), this property 
already ensures that the algorithm is robust to noise without any further 
changes. This is because the noise injected to each $\bx_t$ merely causes the 
exact gradient estimate to change to a random vector which is correct in 
expectation: If we assume $\ell$ is a classification loss, then
\begin{align*}
&\E\left[\ell'(y_t\inner{\bw_t,\Psi(\tilde{\bx}_t)})\Psi(\tilde{\bx}_t)\right] 
~=~ \E\left[c\tilde{\bx}_t\right] = \bx_t.
\end{align*}
 
On the other hand, when we use nonlinear kernels and nonlinear loss functions,
% these equalities almost never hold, so
using standard online gradient descent leads to systematic and unknown biases
(since the noise distribution is unknown), which prevents the method from
working properly. To deal with this problem, we now turn to describe a technique
for estimating expressions 
such as $\ell'\bigl(y_t\inner{\bw_t,\Psi(\bx_t)}\bigr)$
in an unbiased manner. 
In \subsecref{subsec:featuremapping}, we discuss how $\Psi(\bx_t)$ can be 
estimated in an unbiased manner.

%%%%%%%%%%%%%%%%%%%%%%%%%%%%%%%%%%%%%%%%%%%%%%%%%%
\subsection{Unbiased Estimators for Non-Linear Functions}\label{subsec:indian}
Suppose that we are given access to independent copies of a real random
variable $X$, with expectation $\E[X]$, and some real function $f$, and we wish to
construct an unbiased estimate of $f(\E[X])$. If $f$ is a linear function, then
this is easy: just sample $x$ from $X$, and return $f(x)$. By linearity,
$\E[f(X)]=f(\E[X])$ and we are done.
The problem becomes less trivial when $f$ is a general, non-linear function,
since usually $\E[f(X)]\neq f(\E[X])$. In fact, when $X$ takes finitely many
values and $f$ is not a polynomial function, one can prove that no unbiased
estimator can exist (see \cite{Paninski03}, Proposition~8 and its proof).
Nevertheless, we show how in many cases one can construct an unbiased
estimator of $f(\E[X])$, including cases covered by the impossibility result.
There is no contradiction, because we do not construct a ``standard'' estimator.
Usually, an estimator is a function from a given sample to the range of the
parameter we wish to estimate. An implicit assumption is that the size of the
sample given to it is fixed, and this is also a crucial ingredient in the
impossibility result. We circumvent this by constructing an estimator based
on a random number of samples.

Here is the key idea: suppose $f:\reals\rightarrow\reals$ is any function 
continuous on a bounded interval. It is well known that one can construct a 
sequence of polynomials $(Q_{n}(\cdot))_{n=1}^{\infty}$, where $Q_{n}(\cdot)$ 
is a polynomial of degree $n$, which converges uniformly to $f$ on the 
interval. If $Q_n(x)=\sum_{i=0}^{n}\gamma_{n,i} x^i$, let 
$Q'_n(x_1,\ldots,x_n)=\sum_{i=0}^{n}\gamma_{n,i}\prod_{j=1}^{i}x_j$. Now, 
consider the estimator which draws a positive integer $N$ according to some 
distribution $\Pr(N=n)=p_n$, samples $X$ for $N$ times to get 
$x_1,x_2,\ldots,x_N$, and returns $ 
\tfrac{1}{p_N}\left(Q'_N(x_1,\ldots,x_N)-Q'_{N-1}(x_1,\ldots,x_{N-1})\right), 
$ where we assume $Q'_{0}=0$. The expected value of this estimator is equal to:
\begin{align*}
&\E_{N,x_1,\ldots,x_N}
\left[\frac{1}{p_N}
\left(Q'_N(x_1,\ldots,x_N)-Q'_{N-1}(x_1,\ldots,x_{N-1})\right)\right] \\
&=
\sum_{n=1}^{\infty}\frac{p_n}{p_n}\E_{x_1,\ldots,x_n}
\left[Q'_n(x_1,\ldots,x_n)-Q'_{n-1}(x_1,\ldots,x_{n-1})\right]\\
&=\sum_{n=1}^{\infty}\bigl(Q_{n}(\E[X])-Q_{n-1}(\E[X])\bigr) = f(\E[X]).
\end{align*}
Thus, we have an unbiased estimator of $f(\E[X])$. 

This technique appeared in a rather obscure early 1960's paper \cite{singh64}
from sequential estimation theory, and appears to be little known, particularly
outside the sequential estimation community.  However, we believe this
technique is interesting, and expect it to have useful applications for other
problems as well.

While this may seem at first like a very general result, the variance of this
estimator must be bounded for it to be useful. Unfortunately, this is not true
for general continuous functions. More precisely, let $N$ be distributed
according to $p_n$, and let $\theta$ be the value returned by the estimator. In
\cite{BhandariBose1990}, it is shown that if $X$ is a Bernoulli random
variable, and if $\E[\theta N^k]<\infty$ for some integer $k\geq 1$, then $f$
must be $k$ times continuously differentiable. Since $\E[\theta N^k]\leq
(\E[\theta^2]+\E[N^{2k}])/2$, this means that functions $f$ which yield an
estimator with finite variance, while using a number of queries with bounded
variance, must be continuously differentiable. Moreover, in case we desire the
number of queries to be essentially constant (i.e.  choose a distribution for
$N$ with exponentially decaying tails), we must have $\E[N^{k}]<\infty$ for all
$k$, which means that $f$ should be infinitely differentiable (in fact,
in~\cite{BhandariBose1990} it is conjectured that $f$ must be analytic in such
cases).

Thus, we focus in this paper on functions $f$ which are analytic, i.e., they
can be written as $f(x)=\sum_{i=0}^{\infty}\gamma_i x^i$ for appropriate
constants $\gamma_0,\gamma_1,\ldots$. In that case, $Q_n$ can simply be the
truncated Taylor expansion of $f$ to order $n$, i.e.,
$Q_n=\sum_{i=0}^{n}\gamma_i x^i$. Moreover, we can pick $p_n \propto 1/p^n$ for
any $p>1$. So the estimator becomes the following: we sample a nonnegative
integer $N$ according to $\Pr(N=n)=(p-1)/p^{n+1}$, sample $X$ independently $N$
times to get $x_1,x_2,\ldots,x_N$, and return $\theta = \gamma_N
\frac{p^{N+1}}{p-1}x_1x_2\cdots x_N$ where we set $\theta =
\tfrac{p}{p-1}\gamma_0$ if $N=0$.\footnote{ Admittedly, the event $N=0$ should
  receive zero probability, as it amounts to ``skipping'' the sampling
  altogether. However, setting $\Pr(N=0)=0$ appears to improve the
  bound in this paper only in the smaller order terms, while making the
  analysis in the paper more complicated.  } We have the following:
\begin{lemma}\label{lem:unb_est}
  For the above estimator, it holds that $\E[\theta]=f(\E[X])$. The
  expected number of samples used by the estimator is $1/(p-1)$, and the
  probability of it being at least $z$ is $p^{-z}$. Moreover, if
  we assume that $f_{+}(x)=\sum_{n=0}^{\infty} |\gamma_n|x^n$
  exists for any $x$ in the domain of interest, then
\[
\E[\theta^2] \leq \frac{p}{p-1}f_{+}^2\left(\sqrt{p\E[X^2]}\right).
\]
\end{lemma}
\begin{proof}
  The fact that $\E[\theta]=f(\E[X])$ follows from the discussion above. The
  results about the number of samples follow directly from properties of the
  geometric distribution. As for the second moment, $\E[\theta^2]$ equals
\begin{align*}
&\E_{N,x_1,\ldots,x_N}
\left[\gamma_N^2\frac{p^{2(N+1)}}{(p-1)^2}x_1^2x_2^2\cdots x_N^2\right]
~=~
\sum_{n=0}^{\infty}\frac{(p-1)p^{2(n+1)}}{(p-1)^2p^{n+1}}
\gamma_n^2\E_{x_1,\ldots,x_n}\left[x_1^2x_2^2\cdots x_n^2\right]\\
&= \frac{p}{p-1}\sum_{n=0}^{\infty}\gamma_n^2p^{n}\left(\E[X^2]\right)^n
~=~ \frac{p}{p-1}\sum_{n=0}^{\infty}
\left(|\gamma_n|\left(\sqrt{p\E[X^2]}\right)^n\right)^2\\
&\leq \frac{p}{p-1}
\left(\sum_{n=0}^{\infty}|\gamma_n|\left(\sqrt{p\E[X^2]}\right)^n\right)^2
~=~ \frac{p}{p-1}f_{+}^2\left(\sqrt{p\E[X^2]}\right).
\end{align*}
\end{proof}
The parameter $p$ provides a \emph{tradeoff} between the variance of the
estimator and the number of samples needed: the larger is $p$, the less samples
do we need, but the estimator has more variance. In any case, the sample size
distribution decays exponentially fast, so the sample size is essentially
bounded.

It should be emphasized that the estimator associated with 
\lemref{lem:unb_est} is tailored for generality, and is suboptimal in some 
cases. For example, if $f$ is a polynomial function, then $\gamma_n=0$ for 
sufficiently large $n$, and there is no reason to sample $N$ from a 
distribution supported on all nonnegative integers - it just increases the 
variance. Nevertheless, in order to keep the presentation unified and general, 
we will always use this type of estimator. If needed, the estimator can always 
be optimized for specific cases.

We also note that this technique can be improved in various directions, if 
more is known about the distribution of $X$. For instance, if we have some 
estimate of the expectation and variance of $X$, then we can perform a Taylor 
expansion around the estimated $\E[X]$ rather than $0$, and tune the 
probability distribution of $N$ to be different than the one we used above. 
These modifications can allow us to make the variance of the estimator 
arbitrarily small, if the variance of $X$ is small enough. Moreover, one can 
take polynomial approximations to $f$ which are perhaps better than truncated 
Taylor expansions. In this paper, for simplicity, we will ignore these 
potential improvements.

Finally, we note that a related result in~\cite{BhandariBose1990} implies that
it is impossible to estimate $f(\E[X])$ in an unbiased manner when $f$ is
discontinuous, even if we allow a number of queries and estimator values which
are infinite in expectation. Therefore, since the derivative of the hinge
loss is not continuous, estimating in an unbiased manner the gradient of the
hinge loss with arbitrary noise appears to be impossible. Thus, if online
learning with noise and hinge loss is at all feasible, a rather different
approach than ours will need to be taken.

\subsection{Unbiasing Noise in the RKHS}\label{subsec:featuremapping}
The third component of our approach involves the unbiased estimation of 
$\Psi(\bx_t)$, when we only have unbiased noisy copies of $\bx_t$. Here again, 
we have a non-trivial problem, because the feature mapping $\Psi$ is usually 
highly non-linear, so $\E[\Psi(\tilde{\bx}_t)]\neq \Psi(\E[\tilde{\bx}_t])$ in 
general. Moreover, $\Psi$ is not a scalar function, so the technique of 
\subsecref{subsec:indian} will not work as-is.

To tackle this problem, we construct an explicit feature mapping, which needs 
to be tailored to the kernel we want to use. To give a very simple example, 
suppose we use the homogeneous 2nd-degree polynomial kernel, 
$k(\mathbf{r},\mathbf{s})=(\inner{\mathbf{r},\mathbf{s}})^2$. It is not hard 
to verify that the function $\Psi:\reals^d\mapsto \reals^{d^2}$, defined via 
$\Psi(\bx)=(x_1 x_1,x_1 x_2,\ldots,x_d x_d)$, is an explicit feature mapping 
for this kernel. Now, if we query two independent noisy copies 
$\tilde{\bx},\tilde{\bx}'$ of $\bx$, we have that the expectation of the 
random vector $(\tilde{x}_1 \tilde{x}'_1,\tilde{x}_1 \tilde{x}'_2,\ldots, 
\tilde{x}_d \tilde{x}'_d)$ is nothing more than $\Psi(\bx)$. Thus, we can 
construct unbiased estimates of $\Psi(\bx)$ in the RKHS. Of course, this 
example pertains to a very simple RKHS with a finite dimensional 
representation. By a randomization trick somewhat similar to the one in 
\subsecref{subsec:indian}, we can adapt this approach to infinite dimensional 
RKHS as well. In a nutshell, we represent $\Psi(\bx)$ as an 
infinite-dimensional vector, and its noisy unbiased estimate is a vector which 
is non-zero on only finitely many entries, using finitely many noisy queries. 
Moreover, inner products between these estimates can be done efficiently, 
allowing us to implement the learning algorithms, and use the resulting 
predictor on test instances.

%%%%%%%%%%%%%%%%%%%%%%%%%%%%%%%%%%%%%%%%%%%%%%%%%%
%%%%%%%%%%%%%%%%%%%%%%%%%%%%%%%%%%%%%%%%%%%%%%%%%%
%%%%%%%%%%%%%%%%%%%%%%%%%%%%%%%%%%%%%%%%%%%%%%%%%%
\section{Main Results}\label{sec:results}

%%%%%%%%%%%%%%%%%%%%%%%%%%%%%%%%%%%%%%%%%%%%%%%%%%
\subsection{Algorithm}\label{subsec:main_algorithm}
We present our algorithmic approach in a modular form. We start by introducing
the main algorithm, which contains several subroutines. Then we prove our two
main results, which bound the regret of the algorithm, the number of queries to
the oracle, and the running time for two types of kernels: dot product and 
Gaussian (our results can be extended to other kernel types as well). In 
itself, the algorithm is nothing more than a standard online gradient descent 
algorithm with a standard $O(\sqrt{T})$ regret bound. Thus, most of the proofs 
are devoted to a detailed discussion of how the subroutines are implemented 
(including explicit pseudo-code). In this section, we just describe one 
subroutine, based on the techniques discussed in \secref{sec:techniques}. The 
other subroutines require a more detailed and technical discussion, and thus 
their implementation is described as part of the proofs in 
\secref{sec:proofs}. In any case, the intuition behind the implementations and 
the techniques used are described in \secref{sec:techniques}.

For simplicity, we will focus on a finite-horizon setting, where the number of 
online rounds $T$ is fixed and known to the learner. The algorithm can easily 
be modified to deal with the infinite horizon setting, where the learner needs 
to achieve sub-linear regret for all $T$ simultaneously. Also, for the 
remainder of this subsection, we assume for simplicity that $\ell$ is a 
classification loss, namely can be written as a function of 
$\ell(y\inner{\bw,\Psi(\bx)})$. It is not hard to adapt the results below to 
the case where $\ell$ is a regression loss (where $\ell$ is a function of 
$\inner{\bw,\Psi(\bx)}-y$).

We note that at each round, the algorithm below constructs an object which we
denote as $\tilde{\Psi}(\bx_t)$. This object has two interpretations here:
formally, it is an element of a reproducing kernel Hilbert space (RKHS)
corresponding to the kernel we use, and is equal in expectation to
$\Psi(\bx_t)$. However, in terms of implementation, it is simply a data
structure consisting of a finite set of vectors from $\reals^d$. Thus, it can be
efficiently stored in memory and handled even for infinite-dimensional RKHS.

\begin{Algorithm}%[t]
 \caption{Kernel Learning Algorithm with Noisy Input} \label{algo:kernel}
    \begin{Balgorithm}
      \texttt{Parameters:} Learning rate $\eta>0$, number of rounds $T$, sample parameter $p>1$.\\
      \texttt{Initialize:} \+\\
      $\alpha_{i}=0$ for all $i=1,\ldots,T$.\\
      $\tilde{\Psi}(\bx_i)$ for all $i=1,\ldots,T$ \+\\
      //~$\tilde{\Psi}(\bx_i)$ is a data structure which can store a variable number of vectors in $\reals^d$\-\-\\
      \texttt{For $t=1\ldots T$}\+\\
      Define $\bw_t = \sum_{i=1}^{t-1}\alpha_i\tilde{\Psi}(\bx_i)$\\
      Receive $A_t,y_t$\hspace{3.7 cm} //~The oracle $A_t$ provides noisy estimates of $\bx_t$\\
      Let $\tilde{\Psi}(\bx_t):=\texttt{Map\_Estimate}(A_t,p)$~~~~~
      //~Get unbiased estimate of $\Psi(\bx_t)$ in the RKHS\\
      Let $\tilde{g}_t:=\texttt{Grad\_Length\_Estimate}(A_t,y_t,p)$~ //~Get
      unbiased estimate of
      $\ell'(y_t\inner{\bw_t,\Psi(\bx_t)})$\\
      Let $\alpha_{t}:=-\tilde{g}_t\eta/\sqrt{T}$\hspace{2.9 cm}//~Perform gradient step\\
      Let $\tilde{n}_t:=\sum_{i=1}^{t}\sum_{j=1}^{t}\alpha_{t,i} \alpha_{t,j}
      \texttt{Prod}(\tilde{\Psi}(\bx_i),\tilde{\Psi}(\bx_j))$~\+\+\\
      //~Compute squared norm, where $\texttt{Prod}(\tilde{\Psi}(\bx_i),\tilde{\Psi}(\bx_j))$ returns $\inner{\tilde{\Psi}(\bx_i),\tilde{\Psi}(\bx_j)}$\-\-\\
      If $\tilde{n}_t> B_{w}$\hspace{4.2 cm}//~If norm squared is larger than $B_{w}$, then project \+\\
      Let $\alpha_{i}:=\alpha_{i}\frac{\sqrt{B_w}}{\tilde{n}_t}$
      for all $i=1,\ldots,t$\\
       %Let $\alpha_{i}:=\frac{1}{T}\sum_{t}\alpha_{t,i}$ for all 
       %$i=1,\ldots,T$\\
       %Return $(\alpha_{1},\ldots,\alpha_{T})$.
    \end{Balgorithm}
\end{Algorithm}

Like $\tilde{\Psi}(\bx_t)$, $\bw_{t+1}$ has also two interpretations: formally, it
is an element in the RKHS, as defined in the pseudocode. In terms of
implementation, it is defined via the data structures
$\tilde{\Psi}(\bx_1),\ldots,\tilde{\Psi}(\bx_{t})$ and the values of
$\alpha_1,\ldots,\alpha_t$ at round $t$. To apply this hypothesis on a given
instance $\bx$, we compute
$\sum_{i=1}^{t}\alpha_{t,i}\texttt{Prod}(\tilde{\Psi}(\bx_i),\bx')$, where
$\texttt{Prod}(\tilde{\Psi}(\bx_i),\bx')$ is a subroutine which returns
$\inner{\tilde{\Psi}(\bx_i),\Psi(\bx')}$ (a pseudocode is provided as part of the proofs later on).

We now turn to the main results pertaining to the algorithm. The first result 
shows what regret bound is achievable by the algorithm for any dot-product 
kernel, as well as characterize the number of oracle queries per instance, and 
the overall running time of the algorithm. 
\begin{theorem}\label{thm:main_inner_product}
Assume that the loss function $\ell$ has an analytic derivative
$\ell'(a)=\sum_{n=0}^{\infty}\gamma_n a^n$ for all $a$ in its
domain, and let $\ell'_{+}(a)=\sum_{n=0}^{\infty}|\gamma_n| a^n$ (assuming it
exists). Assume also that the kernel $k(\bx,\bx')$ can be written as
$Q(\inner{\bx,\bx'})$ for all $\bx,\bx'\in \reals^d$.
% and arbitrary dimensionality.
Finally, assume that $\E[\norm{\tilde{\bx}_t}^2]\leq 
B_{\tilde{\bx}}$ for any $\tilde{\bx}_t$ returned by the oracle at round $t$, 
for all $t=1,\ldots,T$.
Then, for all $B_{\bw}>0$ and $p>1$, it is possible to implement the
subroutines of \algoref{algo:kernel} such that:
\begin{itemize}
\item The expected number of queries to each oracle $A_t$ is $\frac{p}{(p-1)^2}$.
\item The expected running time of the algorithm is
$O\left(T^3\left(1+\frac{dp}{(p-1)^2}\right)\right)$.
\item If we run \algoref{algo:kernel} with
  $\eta=B_{\bw}\big/\sqrt{u}\ell'_{+}\bigl(\sqrt{(p-1)u}\bigr)$,
  where $u = B_{\bw}\left(\frac{p}{p-1}\right)^2Q(pB_{\tilde{\bx}})$, then
\[
\E\left[\sum_{t=1}^T \ell(y_t\inner{\bw_t,\Psi(\bx_t)}) - 
\min_{\bw \,:\, \norm{\bw}^2\leq B_{\bw}}
\sum_{t=1}^T \ell(y_t\inner{\bw,\Psi(\bx_t)})\right] 
~\leq~ 
\ell'_{+}\bigl(\sqrt{(p-1)u}\bigr)\sqrt{uT}~.
\]
\end{itemize}
The expectations are with respect to the randomness of the oracles and the
algorithm throughout its run.
\end{theorem}
We note that the distribution of the number of oracle queries can be specified
explicitly, and it decays very rapidly - see the proof for details. Also, for
simplicity, we only bound the expected regret in the theorem above. If the
noise is bounded almost surely or with sub-Gaussian tails (rather than just
bounded variance), then it is possible to obtain similar guarantees with high
probability, by relying on Azuma's inequality or variants thereof (see for
example~\cite{CesaBianchiCoGe04}).

We now turn to the case of Gaussian kernels.
\begin{theorem}\label{thm:main_Gaussian}
Assume that the loss function $\ell$ has an analytic derivative
$\ell'(a)=\sum_{n=0}^{\infty}\gamma_n a^n$ for all $a$ in its
domain, and let $\ell'_{+}(a)=\sum_{n=0}^{\infty}|\gamma_n| a^n$ (assuming it
exists). Assume that the kernel $k(\bx,\bx')$ is defined as
$\exp(-\norm{\bx-\bx}^2/\sigma^2)$. Finally, assume that
$\E[\norm{\tilde{\bx}_t}^2]\leq B_{\tilde{\bx}}$ for any $\tilde{\bx}_t$
returned by the oracle at round $t$, for all $t=1,\ldots,T$.
Then for all $B_{\bw}>0$ and $p>1$ it is possible to implement the
subroutines of \algoref{algo:kernel} such that 
\begin{itemize}
\item The expected number of queries to each oracle $A_t$ is $\frac{3p}{(p-1)^2}$.
\item The expected running time of the algorithm is
  $O\left(T^3\left(1+\frac{dp}{(p-1)^2}\right)\right)$.
\item If we run \algoref{algo:kernel} with
$\eta=B_{\bw}\big/\sqrt{u}\ell'_{+}\bigl(\sqrt{(p-1)u}\bigr)$,
where
\[
u = B_{\bw}\left(\frac{p}{p-1}\right)^3
\exp\left(\frac{\sqrt{p}B_{\tilde{\bx}}+2p\sqrt{B_{\tilde{\bx}}}}{\sigma^2}\right)
\]
then
\[
\E\left[\sum_{t=1}^T \ell(y_t\inner{\bw_t,\Psi(\bx_t)}) - 
\min_{\bw \,:\, \norm{\bw}^2\leq B_{\bw}}
\sum_{t=1}^T \ell(y_t\inner{\bw,\Psi(\bx_t)})\right] 
~\leq~ 
\ell'_{+}(\sqrt{(p-1)u})\sqrt{uT}~.
\]
\end{itemize}
The expectations are with respect to the randomness of the oracles and the
algorithm throughout its run.
\end{theorem}
As in \thmref{thm:main_inner_product}, note that the number of oracle queries
has a fast decaying distribution. Also, note that with Gaussian kernels,
$\sigma^2$ is usually chosen to be on the order of the example's squared
norms. Thus, if the noise added to the examples is proportional to their
original norm, we can assume that $B_{\tilde{\bx}}/\sigma^2=O(1)$, and thus $u$
which appears in the bound is also bounded by a constant.

As previously mentioned, most of the subroutines are described in the proofs 
section, as part of the proof of \thmref{thm:main_inner_product}. Here, we only 
show how to implement \texttt{Grad\_Length\_Estimate} subroutine, which 
returns the gradient length estimate $\tilde{g}_t$. The idea is based on the 
technique described in \subsecref{subsec:indian}. We prove that $\tilde{g}_t$ 
is an unbiased estimate of $\ell'(y_t\inner{\bw_t,\Psi(\bx_t)})$, and bound 
$\E[\tilde{g}^2_t]$. As discussed earlier, we assume that $\ell'(\cdot)$ is 
analytic and can be written as $\ell'(a)=\sum_{n=0}^{\infty}\gamma_n a^n$.

\begin{Subroutine}%[t]
 \caption{\texttt{Grad\_Length\_Estimate}$(A_t,y_t,p)$} \label{sub:grad}
    \begin{Balgorithm}
      Sample nonnegative integer $n$ according to $\Pr(n)=(p-1)/p^{n+1}$\\
      \texttt{For} $j=1,\ldots,n$\+\\
      Let $\tilde{\Psi}(\bx_t)_j:=\texttt{Map\_Estimate}(A_t)$~~~//~Get unbiased estimate of $\Psi(\bx_t)$ in the RKHS\-\\
      Return
      $\tilde{g}_t:=y_t\gamma_n\frac{p^{n+1}}{p-1}\prod_{j=1}^{n}\left(\sum_{i=1}^{t-1}
        \alpha_{t-1,i}\texttt{Prod}(\tilde{\Psi}(\bx_i),\tilde{\Psi}(\bx_t)_j)\right)$\\
    \end{Balgorithm}
\end{Subroutine}

\begin{lemma}\label{lem:grad_unbiased}
Assume that $\E[\tilde{\Psi}(\bx_t)]=\Psi(\bx_t)$, and that 
\texttt{Prod}$(\tilde{\Psi}(\bx),\tilde{\Psi}(\bx'))$ returns 
$\inner{\tilde{\Psi}(\bx),\tilde{\Psi}(\bx')}$ for all $\bx,\bx'$. 
Then for any given $\bw_t = \alpha_{t-1,1}\tilde{\Psi}(\bx_1)+\cdots+\alpha_{t-1,t-1}\tilde{\Psi}(\bx_{t-1})$ it holds that
\[
\E_t[\tilde{g}_t] = y_t\ell'(y_t\inner{\bw_t,\Psi(\bx_t)})
\qquad\text{and}\qquad
\E_t[\tilde{g}^2_t] \leq 
\frac{p}{p-1}\ell_{+}^{'}\left(\sqrt{pB_{\bw}B_{\tilde{\Psi}(\bx)}}\right)^2
\]
where the expectation is with respect to the randomness of \subref{sub:grad}, 
and $\ell'_{+}(a)=\sum_{n=0}^{\infty}|\gamma_n|a^n$.
\end{lemma}
\begin{proof}
The result follows from \lemref{lem:unb_est}, where $\tilde{g}_t$ corresponds 
to the estimator $\theta$, the function $f$ corresponds to $\ell'$, and the 
random variable $X$ corresponds to $\inner{\bw_t,\tilde{\Psi}(\bx_t)}$ (where 
$\tilde{\Psi}(\bx_t)$ is random and $\bw_t$ is held fixed). The term $\E[X^2]$ 
in \lemref{lem:unb_est} can be upper bounded as
\[
    \E_t\left[\bigl(\inner{\bw_t,\tilde{\Psi}(\bx_t)}\bigr)^2\right]
\leq 
%   \E\left[\norm{\bw_t}^2\right]
    \norm{\bw_t}^2\,
    \E_t\left[\norm{\tilde{\Psi}(\bx_t)}^2\right]
\leq 
    B_{\bw}B_{\tilde{\Psi}(\bx)}~.
\]
\end{proof}

%%%%%%%%%%%%%%%%%%%%%%%%%%%%%%%%%%%%%%%%%%%%%%%%%%
\subsection{Loss Function Examples}\label{subsec:examples}

Theorems~\ref{thm:main_inner_product} and~\ref{thm:main_Gaussian} both deal 
with generic loss functions $\ell$ whose derivative can be written as 
$\sum_{n=0}^{\infty}\gamma_na^n$, and the regret bounds involve the functions 
$\ell'_{+}(a)=\sum_{n=0}^{\infty}|\gamma_n|a^n$. Below, we present a few 
examples of loss functions and their corresponding $\ell'_{+}$. As mentioned 
earlier, while the theorems in the previous subsection are in terms of 
classification losses (i.e., $\ell$ is a function of 
$y\inner{\bw,\Psi(\bx)}$), virtually identical results can be proven for 
regression losses (i.e., $\ell$ is a function of $\inner{\bw,\Psi(\bx)}-y$), 
so we will give examples from both families. Working out the first two 
examples is straightforward. The proofs of the other two appear in 
\secref{sec:proofs}. The loss functions are illustrated graphically in 
\figref{fig:smoothed_losses}.
\begin{example}
For the squared loss function, 
$\ell(\inner{\bw,\bx},y)=(\inner{\bw,\bx}-y)^2$, we have 
$\ell'_{+}\bigl(\sqrt{(p-1)u)}\bigr)=2\sqrt{(p-1)u}$.
\end{example}
\begin{example}
  For the exponential loss function,
  $\ell(\inner{\bw,\bx},y)=e^{y\inner{\bw,\bx}}$,
  we have $\ell'_{+}\bigl(\sqrt{(p-1)u}\bigr)=e^{\sqrt{(p-1)u}}$.
\end{example}
\begin{example}\label{ex:smooth_abs}
Consider a ``smoothed'' absolute loss function 
$\ell_{\sigma}(\inner{\bw,\Psi(\bx)}-y)$, defined as an antiderivative of 
$\mathrm{Erf}(sa)$ for some $s>0$ (see proof for exact analytic form). Then we 
have that $\ell'_{+}\bigl(\sqrt{(p-1)u}\bigr)\leq 
\frac{1}{2}+\frac{1}{s\sqrt{\pi(p-1)u}}\left(e^{s^2(p-1)u}-1\right)$.
\end{example}
\begin{example}\label{ex:smooth_hinge}
Consider a ``smoothed'' hinge loss $\ell(y\inner{\bw,\Psi(\bx)})$, defined as an 
antiderivative of $(\mathrm{Erf}(s(a-1))-1)/2$ for some $s>0$ (see proof for 
exact analytic form). Then we have that $\ell'_{+}\bigl(\sqrt{(p-1)u}\bigr)\leq 
\frac{2}{s\sqrt{\pi(p-1)u}}\left(e^{s^2(p-1)u-1}\right)$. 
\end{example}
For any $s$, the loss function in the last two examples are convex, and 
respectively approximate the absolute loss $\big|\inner{\bw,\Psi(\bx)}-y\big|$
and the hinge loss $\max\bigl\{0,1-y\inner{\bw,\Psi(\bx)}\bigr\}$ arbitrarily
well for large enough $s$. 
\figref{fig:smoothed_losses} shows these loss functions graphically for $s=1$. 
Note that $s$ need not be large in order to get a good approximation. Also, we 
note that both the loss itself and its gradient are computationally easy to
evaluate.

\begin{figure}
\begin{center}
\includegraphics[scale=0.5]{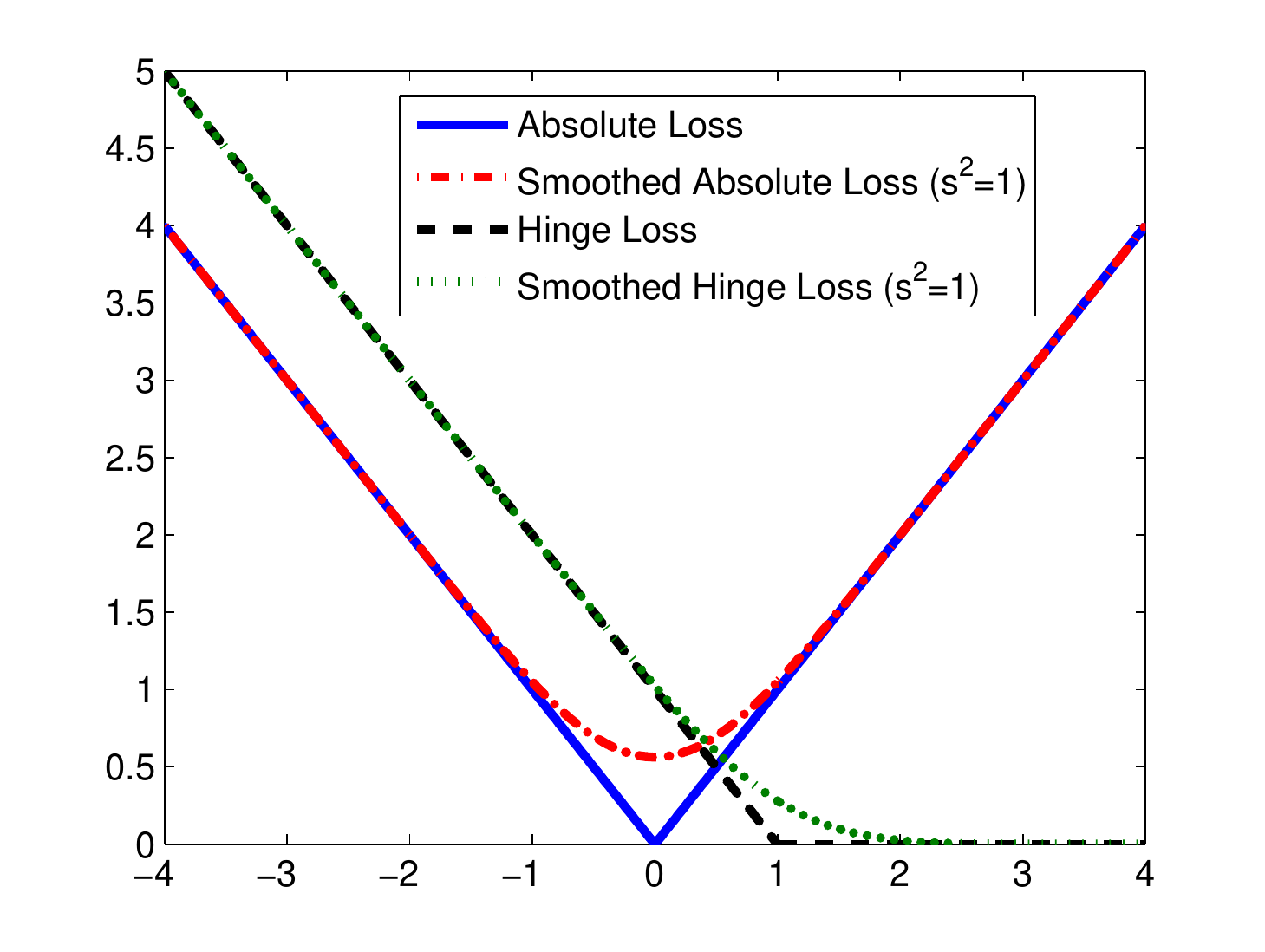}
\end{center}
\vspace{-5mm}
\caption{
Absolute loss, hinge loss, and smooth approximations}\label{fig:smoothed_losses}
\end{figure}

Finally, we remind the reader that as discussed in \subsecref{subsec:indian},
performing an unbiased estimate of the gradient for non-differentiable losses
directly (such as the hinge loss or absolute loss) appears to be impossible in
general. On the flip side, if one is willing to use a random number of queries
with polynomial rather than exponential tails, then one can achieve much better
sample complexity results, by focusing on loss functions (or approximations
thereof) which are only differentiable to a bounded order, rather than fully
analytic. This again demonstrates the tradeoff between the sample size and the
amount of information that needs to be gathered on each training example.

%%%%%%%%%%%%%%%%%%%%%%%%%%%%%%%%%%%%%%%%%%%%%%%%%%
\subsection{One Noisy Copy is Not Enough}
The previous results might lead one to wonder whether it is really necessary to
query the same instance more than once. In some applications this is
inconvenient, and one would prefer a method which works when just a single
noisy copy of each instance is made available. In this subsection we show that,
unfortunately, such a method cannot be found. Specifically, we prove that under
very mild assumptions, no method can achieve sub-linear regret when it has
access to just a single noisy copy of each instance. On the other hand, for the
case of squared loss and linear kernels, our techniques can be adapted to work
with exactly two noisy copies of each instance,\footnote{ In a nutshell, for
  squared loss and linear kernels, we just need to estimate
  $2(\inner{\bw_t,\bx_t}-y_t)\bx_t$ in an unbiased manner at each round $t$.
  This can be done by computing
  $2(\inner{\bw_t,\tilde{\bx}_t}-y_t)\tilde{\bx}'_t$, where
  $\tilde{\bx}_t,\tilde{\bx}'_t$ are two noisy copies of $\bx_t$.  } so without
further assumptions, the lower bound that we prove here is indeed tight. For
simplicity, we prove the result for linear kernels (i.e., where
$k(\bx,\bx')=\inner{\bx,\bx'}$).  It is an interesting open problem to show
improved lower bounds when nonlinear kernels are used. We also note that the
result crucially relies on the learner not knowing the noise distribution, and
we leave to future work the investigation of what happens when this assumption
is relaxed.
\begin{theorem}\label{thm:impossibility2}
  Let $\Wcal$ be a compact convex subset of $\reals^d$, and let
  $\ell(\cdot,1):\reals\mapsto\reals$ satisfies the following:
  (1) it is bounded from below; (2) it is differentiable at $0$
  with $\ell'(0,1)<0$.
  For any learning algorithm which selects hypotheses from $\Wcal$ and
  is allowed access to a single noisy copy of the instance at each round $t$,
  there exists a strategy for the adversary such that the sequence
  $\bw_1,\bw_2,\dots$ of predictors output by the algorithm satisfies
\[
    \limsup_{T\rightarrow \infty} \max_{\bw\in \Wcal}
    \frac{1}{T}\sum_{t=1}^{T}\Bigl( \ell(\inner{\bw_t,\bx_t},y_t)
    - \ell(\inner{\bw,\bx_t},y_t) \Bigr) > 0
\]
with probability 1 with respect to the randomness of the oracles.
\end{theorem}
Note that condition (1) is satisfied by virtually any loss function other
than the linear loss, while condition (2) is satisfied by most regression
losses, and by all \emph{classification calibrated losses}, which include all
reasonable losses for classification (see~\cite{BartlettJorMca06}). The most
obvious example where the conditions are not satisfied is when $\ell(\cdot,1)$
is a linear function. This is not surprising, because when $\ell(\cdot,1)$
is linear, the learner is always robust to noise
(see the discussion at~\secref{sec:techniques}).

The intuition of the proof is very simple: the adversary chooses beforehand
whether the examples are drawn i.i.d.\ from a distribution $\Dcal$, and then
perturbed by noise, or drawn i.i.d.\ from some other distribution $\Dcal'$
without adding noise. The distributions $\Dcal,\Dcal'$ and the noise are
designed so that the examples observed by the learner are distributed in the
same way irrespective to which of the two sampling strategies the adversary
chooses. Therefore, it is impossible for the learner accessing a single copy of
each instance to be statistically consistent with respect to both distributions
simultaneously. As a result, the adversary can always choose a distribution on
which the algorithm will be inconsistent, leading to constant regret. The full
proof is presented in Section \ref{subsec:impossibility2}.

%%%%%%%%%%%%%%%%%%%%%%%%%%%%%%%%%%%%%%%%%%%%%%%%%%
%%%%%%%%%%%%%%%%%%%%%%%%%%%%%%%%%%%%%%%%%%%%%%%%%%
%%%%%%%%%%%%%%%%%%%%%%%%%%%%%%%%%%%%%%%%%%%%%%%%%%
\section{Proofs}\label{sec:proofs}
Due to the lack of space, some of the proofs are given in the
\fullpaper{the appendix}\conferencepaper{\cite{CesShalSham10}}. 

\subsection{Preliminary Result}
To prove \thmref{thm:main_inner_product} and \thmref{thm:main_Gaussian}, we
need a theorem which basically states that if all subroutines in algorithm
\ref{algo:kernel} behave as they should, then one can achieve an $O(\sqrt{T})$
regret bound. This is provided in the following theorem, which is an adaptation
of a standard result of online convex optimization (see,
e.g.,~\cite{Zinkevich03}). The proof is given in \fullpaper{\appref{app:ProofKernelMain}}\conferencepaper{\cite{CesShalSham10}}.
\begin{theorem}\label{thm:kernel_main}
Assume the following conditions hold with respect to \algoref{algo:kernel}: 
\begin{enumerate}
\item For all $t$, $\tilde{\Psi}(\bx_t)$ and $\tilde{g}_t$ are
independent of each other (as random variables induced by the randomness of 
\algoref{algo:kernel}) as well as independent of any 
$\tilde{\Psi}(\bx_i)$ and $\tilde{g}_i$ for $i<t$.
\item For all $t$, $\E[\tilde{\Psi}(\bx_t)]=\Psi(\bx_t)$, and there
  exists a constant $B_{\tilde{\Psi}}>0$ such that
  $\E[\norm{\tilde{\Psi}(\bx_t)}^2]\leq B_{\tilde{\Psi}}$.
\item For all $t$, $\E[\tilde{g}_t]=y_t\ell'(y_t\inner{\bw_t,\Psi(\bx_t)})$, and there
exists a constant $B_{\tilde{g}}>0$ such that $\E[\tilde{g}^2_t]\leq 
B_{\tilde{g}}$.
\item For any pair of instances $\bx,\bx'$,
  $\texttt{Prod}(\tilde{\Psi}(\bx),\tilde{\Psi}(\bx'))
  = \inner{\tilde{\Psi}(\bx),\tilde{\Psi}(\bx')}$.
\end{enumerate}
Then if \algoref{algo:kernel} is run with 
$\eta=\sqrt{\frac{B_{\bw}}{B_{\tilde{g}}B_{\tilde{\Psi}}}}$, the 
following inequality holds
\[
    \E\left[\sum_{t=1}^T\ell\bigl(y_t\inner{\bw_t,\Psi(\bx_t)}\bigr)
-
    \min_{\bw \,:\, \norm{\bw}^2 \le B_{\bw}}
    \sum_{t=1}^T\ell\bigl(y_t\inner{\bw,\Psi(\bx_t)}\bigr)\right]
\leq 
    \sqrt{B_{\bw}B_{\tilde{g}}B_{\tilde{\Psi}}T}~.
\]
    where the expectation is with respect to the randomness of 
    the oracles and the algorithm throughout its run.
\end{theorem}

%%%%%%%%%%%%%%%%%%%%%%%%%%%%%%%%%%%%%%%%%%%%%%%%%%

\subsection{Proof of \thmref{thm:main_inner_product}}\label{subsec:innprod}

In this subsection, we present the proof of \thmref{thm:main_inner_product}. We
first show how to implement the subroutines of \algoref{algo:kernel}, and prove
the relevant results on their behavior. Then, we prove the theorem itself.

It is known that for $k(\cdot,\cdot)=Q(\inner{\bx,\bx'})$ to be a valid kernel, it 
is necessary that $Q(\inner{\bx,\bx'})$ can be written as a Taylor expansion 
$\sum_{n=0}^{\infty}\beta_n (\inner{\bx,\bx'})^n$, where $\beta_n \geq 0$ (see 
theorem 4.19 in \cite{ScholkopfSm02}). This makes these types of kernels 
amenable to our techniques.

We start by constructing an explicit feature mapping $\Psi(\cdot)$ 
corresponding to the RKHS induced by our kernel. For any $\bx,\bx'$, we have 
that
\begin{align*}
    k(\bx,\bx')
&=
    \summ_{n=0}^{\infty}\beta_n(\inner{\bx,\bx'})^n
=
    \summ_{n=0}^{\infty}\beta_n\left(\summ_{i=1}^{d}x_i x'_i\right)^n
\\ &=
    \summ_{n=0}^{\infty}\beta_n\summ_{k_1=1}^{d}\cdots\summ_{k_n=1}^{d} 
    x_{k_1}x_{k_2}\cdots x_{k_n}x'_{k_1}x'_{k_2}\cdots x'_{k_n}
\\ &=
    \summ_{n=0}^{\infty}\summ_{k_1=1}^{d}\cdots\summ_{k_n=1}^{d} 
    \left(\sqrt{\beta_n}x_{k_1}x_{k_2}\cdots x_{k_n}\right) 
    \left(\sqrt{\beta_n}x'_{k_1}x'_{k_2}\cdots x'_{k_n}\right).
\end{align*}
This suggests the following feature representation: for any $\bx$, $\Psi(\bx)$ 
returns an infinite-dimensional vector, indexed by $n$ and $k_1,\ldots,k_n \in 
\{1,\ldots,d\}$, with the entry corresponding to $n,k_1,\ldots,k_n$ being 
$\sqrt{\beta_n}x_{k_1}\cdots x_{k_n}$. The dot product between $\Psi(\bx)$ and 
$\Psi(\bx')$ is similar to a standard dot product between two vectors, and by 
the derivation above equals $k(\bx,\bx')$ as required.

We now use a slightly more elaborate variant of our unbiased estimate
technique, to derive an unbiased estimate of $\Psi(\bx)$. First, we sample $N$
according to $\Pr(N=n)=(p-1)/p^{n+1}$. Then, we query the oracle for $\bx$ for $N$
times to get $\tilde{\bx}^{(1)},\ldots,\tilde{\bx}^{(N)}$, and formally define
$\tilde{\Psi}(\bx)$ as
\begin{equation}\label{eq:tildepsi_explicit}
\tilde{\Psi}(\bx)~=~\sqrt{\beta_n}\frac{p^{n+1}}{p-1}\sum_{k_1=1}^{d}\cdots\summ_{k_n=1}^{d} 
\tilde{x}^{(1)}_{k_1}\cdots\tilde{x}^{(n)}_{k_n}\be_{n,k_1,\ldots,k_n}
\end{equation}
where $\be_{n,k_1,\ldots,k_n}$ represents the unit vector in the direction
indexed by $n,k_1,\ldots,k_n$ as explained above. Since the oracle queries are
i.i.d., the expectation of this expression is
\begin{align*}
\summ_{n=0}^{\infty}\frac{p-1}{p^{n+1}} 
\sqrt{\beta_n} \frac{p^{n+1}}{p-1}\summ_{k_1=1}^{d}\!\!\cdots\!\!\summ_{k_n=1}^{d} 
\E\bigl[\tilde{x}^{(1)}_{k_1}\cdots\tilde{x}^{(n)}_{k_n}\bigr]\be_{n,k_1,\ldots,k_n}
= \summ_{n=0}^{\infty}\summ_{k_1=1}^{d}\!\!\cdots\!\!\summ_{k_n=1}^{d} 
\sqrt{\beta_n}x^{(1)}_{k_1}\cdots x^{(n)}_{k_n}\be_{n,k_1,\ldots,k_n}
\end{align*}
which is exactly $\Psi(\bx)$.
We formalize the needed properties of $\tilde{\Psi}(\bx)$ in the following lemma.
\begin{lemma}\label{lem:tilde_unbiased}
Assuming $\tilde{\Psi}(\bx)$ is constructed as in the discussion above, it
holds that $\E[\tilde{\Psi}(\bx)]=\Psi(\bx)$ for any $\bx$. Moreover, if
the noisy samples $\tilde{\bx}_t$ returned by the oracle $A_t$ satisfy
$\E[\norm{\tilde{\bx}_t}^2]\leq B_{\tilde{\bx}}$, then
\[
\E\left[\norm{\tilde{\Psi}(\bx_t)}^2\right]\leq 
\frac{p}{p-1}Q(pB_{\tilde{\bx}})
\]
where we recall that $Q$ defines the kernel by
$k(\bx,\bx')=Q(\inner{\bx,\bx'})$.
\end{lemma}
\begin{proof}
The first part of the lemma follows from the discussion above. As to the second
part, note that by~(\ref{eq:tildepsi_explicit}),
\begin{align*}
&\E\left[\norm{\tilde{\Psi}(\bx_t)}^2\right]
= \E\left[\beta_n \frac{p^{2n+2}}{(p-1)^2}
\sum_{k_1\ldots,k_n=1}^{d}\left(\tilde{x}^{(1)}_{t,k_1}\cdots\tilde{x}^{(N)}_{t,k_n}\right)^2\right]
= \E\left[\beta_n
\frac{p^{2n+2}}{(p-1)^2}\prod_{j=1}^{n}\big\|\tilde{\bx}^{(j)}_t\big\|^2\right]\\
&= \sum_{n=0}^{\infty}\frac{p-1}{p^{n+1}}\beta_n
\frac{p^{2n+2}}{(p-1)^2}\left(\E\bigl[\tilde{\bx}^2_t\bigr]\right)^n
= \frac{p}{p-1}\sum_{n=0}^{\infty}\beta_n \left(p\E\bigl[\tilde{\bx}^2_t\bigr]\right)^n
\leq \frac{p}{p-1}\sum_{n=0}^{\infty}\beta_n \bigl(pB_{\tilde{\bx}}\bigr)^n
= \frac{p}{p-1}Q(pB_{\tilde{\bx}})
\end{align*}
where the second-to-last step used the fact that $\beta_n\geq 0$ for all
$n$.
\end{proof}
Of course, explicitly storing $\tilde{\Psi}(\bx)$ as defined above is
infeasible, since the number of entries is huge. Fortunately, this is not
needed: we just need to store $\tilde{\bx}^{(1)}_t,\ldots,\tilde{\bx}^{(N)}_t$. The
representation above is used implicitly when we calculate dot products between 
$\tilde{\Psi}(\bx)$ and other elements in the RKHS, via the subroutine 
\texttt{Prod}. We note that while $N$ is a random quantity which might be 
unbounded, its distribution decays exponentially fast, so the number of 
vectors to store is essentially bounded. 

After the discussion above, the pseudocode for \texttt{Map\_Estimate} below 
should be self-explanatory. 

\begin{Subroutine}%[t]
 \caption{\texttt{Map\_Estimate}$(A_t,p)$} \label{sub:map}
    \begin{Balgorithm}
       Sample nonnegative integer $N$ according to $\Pr(N=n)=(p-1)/p^{n+1}$\\
       Query $A_t$ for $N$ times to get 
       $\tilde{\bx}^{(1)}_t,\ldots,\tilde{\bx}^{(N)}_t$\\
       Return $\tilde{\bx}^{(1)}_t,\ldots,\tilde{\bx}^{(N)}_t$ as 
       $\tilde{\Psi}(\bx_t)$.
    \end{Balgorithm}
\end{Subroutine}

We now turn to the subroutine $\texttt{Prod}$, which given two elements in the
RKHS, returns their dot product. This subroutine comes in two flavors: either
as a procedure defined over $\tilde{\Psi}(\bx),\tilde{\Psi}(\bx')$ and
returning $\inner{\tilde{\Psi}(\bx),\tilde{\Psi}(\bx')}$
(\subref{sub:prod1}); or as a procedure defined over
$\tilde{\Psi}(\bx),\bx'$ (\subref{sub:prod2}, where the second element
is an explicitely given vector) and returning
$\inner{\tilde{\Psi}(\bx),\Psi(\bx')}$. This second variant of
\texttt{Prod} is needed when we wish to apply the learned predictor on a new
given instance $\bx'$.
 
\begin{Subroutine}%[t]
 \caption{\texttt{Prod}$(\tilde{\Psi}(\bx),\tilde{\Psi}(\bx'))$} \label{sub:prod1}
    \begin{Balgorithm}
       Let $\bx^{(1)},\ldots,\bx^{(n)}$ be the index and vectors comprising $\Psi(\bx)$\\
       Let $\bx'^{(1)},\ldots,\bx'^{(n')}$ be the index and vectors comprising $\Psi(\bx')$\\
       If $n\neq n'$ return $0$, else return 
       $\beta_n
       \frac{p^{2n+2}}{(p-1)^2}\prod_{j=1}^{n}\inner{\tilde{\bx}^{(j)},\tilde{\bx}'^{(j)}}$
    \end{Balgorithm}
 \end{Subroutine}
 
\begin{lemma}\label{lem:prod_correct}
  \texttt{Prod}$(\tilde{\Psi}(\bx),\tilde{\Psi}(\bx'))$ returns
  $\inner{\tilde{\Psi}(\bx)\tilde{\Psi}(\bx')}$.
\end{lemma}
\begin{proof}
Using the formal representation of $\tilde{\Psi}(\bx),\tilde{\Psi}(\bx')$
in~(\ref{eq:tildepsi_explicit}), we have that
  $\inner{\tilde{\Psi}(\bx),\tilde{\Psi}(\bx')}$ is $0$ whenever $n\neq
  n'$ (because then these two elements are composed of different unit vectors
  with respect to an orthogonal basis). Otherwise, we have that
\begin{align*}
    \inner{\tilde{\Psi}(\bx)\tilde{\Psi}(\bx')}
&=
    \beta_n \frac{p^{2n+2}}{(p-1)^2}\summ_{k_1,\ldots,k_n=1}^{d}
    \tilde{x}^{(1)}_{k_1}\cdots\tilde{x}^{(n)}_{k_n}\,\tilde{x}'^{(1)}_{k_1}\cdots\tilde{x}'^{(n)}_{k_n}
\\ &=
    \beta_n
    \frac{p^{2n+2}}{(p-1)^2}\left(\summ_{k_1=1}^{d}\tilde{x}^{(1)}_{k_1}\tilde{x}'^{(1)}_{k_1}\right)
    \cdots \left(\summ_{k_N=1}^{d}\tilde{x}^{(n)}_{k_N}\tilde{x}'^{(n)}_{k_N}\right)
=
    \beta_n \frac{p^{2n+2}}{(p-1)^2}\prod_{j=1}^{N}\left(\inner{\tilde{\bx}^{(j)},\tilde{\bx}'^{(j)}}\right)
\end{align*}
which is exactly what the algorithm returns, hence the lemma follows.
\end{proof}
The pseudocode for calculating the dot product
$\inner{\tilde{\Psi}(\bx),\Psi(\bx')}$ (where $\bx'$ is known) is very
similar, and the proof is essentially the same.
 
\begin{Subroutine}%[t]
   \caption{\texttt{Prod}$(\tilde{\Psi}(\bx),\bx')$} \label{sub:prod2}
   \begin{Balgorithm}
     Let $n,\bx^{(1)},\ldots,\bx^{(n)}$ be the index and vectors comprising $\Psi(\bx)$\\
     Return $\beta_n
     \frac{p^{n+1}}{p-1}\prod_{j=1}^{n}\inner{\tilde{\bx}^{(j)},\bx'}$
   \end{Balgorithm}
\end{Subroutine}

We are now ready to prove \thmref{thm:main_inner_product}. First,
regarding the expected number of queries, notice that to run
\algoref{algo:kernel}, we invoke \texttt{Map\_Estimate} and
\texttt{Grad\_Length\_Estimate} once at round $t$. \texttt{Map\_Estimate} uses
a random number $B$ of queries distributed as $\Pr(B=n)=(p-1)/p^{n+1}$, and
\texttt{Grad\_Length\_Estimate} invokes \texttt{Map\_Estimate} a random number
$C$ of times, distributed as $\Pr(C=n)=(p-1)/p^{n+1}$. The total number of
queries is therefore $\sum_{j=1}^{C+1}B_{j}$, where $B_{j}$ for
all $j$ are i.i.d.\ copies of $B$. The expected value of this expression, using a
standard result on the expected value of a sum of a random number of independent random
variables, is equal to $(1+\E[C])\E[B_{j}]$, or
$\bigl(1+\frac{1}{p-1}\bigr)\frac{1}{p-1}=\frac{p}{(p-1)^2}$.

In terms of running time, we note that the expected running time of
\texttt{Prod} is $O\bigl(1+\frac{d}{p-1}\bigr)$, this because it performs $N$
multiplications of inner products, each one with running time $O(d)$, and
$\E[N] = \tfrac{1}{p-1}$. The expected running time of
\texttt{Map\_Estimate} is $O\bigl(1+\tfrac{1}{p-1}\bigr)$. The expected running time of
\texttt{Grad\_Length\_Estimate} is
$O\bigl(1+\tfrac{1}{p-1}\bigl(1+\tfrac{1}{p-1}\bigr)+T\bigl(1+\tfrac{d}{p-1}\bigr)\bigr)$,
which can be written as $O\bigl(\tfrac{p}{(p-1)^2}+T\bigl(1+\tfrac{d}{p-1}\bigr)\bigr)$. Since
\algoref{algo:kernel} at each of $T$ rounds calls \texttt{Map\_Estimate} once,
\texttt{Grad\_Length\_Estimate} once, \texttt{Prod} for $O(T^2)$ times, and
performs $O(1)$ other operations, we get that the overall runtime is
\[
O\left(T\left(1+\frac{1}{p-1}+\frac{p}{(p-1)^2}+T\left(1+\frac{d}{p-1}\right)+T^2\left(1+\frac{d}{p-1}\right)\right)\right).
\]
Since $\tfrac{1}{p-1}\leq \tfrac{p}{(p-1)^2}$, we can upper bound this by
\[
O\left(T\left(1+\frac{p}{(p-1)^2}+T^2\left(1+\frac{dp}{(p-1)^2}\right)\right)\right)
=
O\left(T^3\left(1+\frac{dp}{(p-1)^2}\right)\right).
\]

The regret bound in the theorem follows from \thmref{thm:kernel_main}, with the
expressions for constants following from \lemref{lem:grad_unbiased},
\lemref{lem:tilde_unbiased}, and \lemref{lem:prod_correct}.

%%%%%%%%%%%%%%%%%%%%%%%%%%%%%%%%%%%%%%%%%%%%%%%%%%
\subsection{Proof Sketch of \thmref{thm:impossibility2}}\label{subsec:impossibility2}

To prove the theorem, we use a more general result which leads to non-vanishing
regret, and then show that under the assumptions of
\thmref{thm:impossibility2}, the result holds. The proof of the result is given
in
\fullpaper{\appref{app:impossibilityproof}}\conferencepaper{\cite{CesShalSham10}}.
\begin{theorem}\label{thm:impossibility}
Let $\Wcal$ be a compact convex subset of $\reals^d$ and pick any learning
algorithm which selects hypotheses from $\Wcal$ and is allowed access to a
single noisy copy of the instance at each round $t$.
If there exists a distribution over a compact subset of $\reals^d$
such that
\begin{equation}\label{eq:impossibility}
\argmin{\bw\in \Wcal}\,\E\bigl[\ell(\inner{\bw,\bx},1)\bigr]~~~\text{and}~~~ 
\argmin{\bw\in \Wcal}\,\ell\bigl(\inner{\bw,\E[\bx]},1\bigr)
\end{equation}
are disjoint, then there exists a strategy for the adversary such 
that the sequence $\bw_1,\bw_2,\dots \in \Wcal$ of predictors output by the
algorithm satisfies
\[
    \limsup_{T\to\infty} \max_{\bw\in \Wcal} \frac{1}{T}
    \sum_{t=1}^{T}\Bigl(
    \ell(\inner{\bw_t,\bx_t},y_t) - \ell(\inner{\bw,\bx_t},y_t)
    \Bigr) > 0
\]
with probability $1$ with respect to the randomness of the oracles.
\end{theorem}
Another way to phrase this theorem is that the regret cannot vanish, if
given examples sampled i.i.d.\ from a distribution, the learning problem is
more complicated than just finding the mean of the data. Indeed, the
adversary's strategy we choose later on is simply drawing and presenting
examples from such a distribution.
Below, we sketch how we use \thmref{thm:impossibility} in order to prove
\thmref{thm:impossibility2}. A full proof is provided in \fullpaper{\appref{app:impossibility2proof}}\conferencepaper{\cite{CesShalSham10}}.

We construct a very simple one-dimensional distribution, which satisfies the
conditions of \thmref{thm:impossibility}: it is simply the uniform distribution
on $\{3\bx,-\bx\}$, where $\bx$ is the vector $(1,0,\ldots,0)$. Thus, it is enough
to show that 
% \[
% \argmin{\bw\in
% \Wcal} \frac{\ell(\inner{\bw,3\bx},1)+\ell(\inner{\bw,-\bx})}{2}~~~\text{and}~~~ 
% \argmin{\bw\in \Wcal} \ell(\inner{\bw,\bx,1})
% \]
% are disjoint. In fact, since the components of $\bw$ that are orthogonal to $\bx$ 
% do not change the objective function, it suffices to look at $\bw$ which are a 
% linear scaling of $\bx$. Thus, it is enough to show that
\begin{equation}\label{eq:1sets}
\argmin{w \,:\, |w|^2 \leq B_{\bw}} \ell(3w,1)+\ell(-w,1)
\qquad\text{and}\qquad
\argmin{w \,:\, |w|^2\leq B_{\bw}} \ell(w,1)
\end{equation}
are disjoint, for some appropriately chosen $B_{\bw}$. Assuming the contrary,
then under the assumptions on $\ell$, we show that the first set in
\eqref{eq:1sets} is inside a bounded ball around the origin, in a way which is
independent of $B_{\bw}$, no matter how large it is. Thus, if we pick $B_{\bw}$
to be large enough, and assume that the two sets in \eqref{eq:1sets} are not
disjoint, then there must be some $w$ such that both $\ell(3w,1)+\ell(-w,1)$
and $\ell(w,1)$ have a subgradient of zero at $w$. However, this can be shown
to contradict the assumptions on $\ell$, leading to the desired result.

%%%%%%%%%%%%%%%%%%%%%%%%%%%%%%%%%%%%%%%%%%%%%%%%%%
%%%%%%%%%%%%%%%%%%%%%%%%%%%%%%%%%%%%%%%%%%%%%%%%%%
%%%%%%%%%%%%%%%%%%%%%%%%%%%%%%%%%%%%%%%%%%%%%%%%%%

%NCB I kept it short: it's a submission.
\section{Future Work}\label{sec:futurework}
There are several interesting research directions worth pursuing in the noisy 
learning framework introduced here. For instance, doing away with 
unbiasedness, which could lead to the design of estimators that are applicable 
to more types of loss functions, for which unbiased estimators may not even 
exist. Also, it would be interesting to show how additional information one 
has about the noise distribution can be used to design improved estimates, 
possibly in association with specific losses or kernels. Another open question 
is whether our lower bound (\thmref{thm:impossibility2}) can be improved when 
nonlinear kernels are used.

\bibliographystyle{plain}
\bibliography{mybib,nic}

\conferencepaper{\end{document}}

%%%%%%%%%%%%%%%%%%%%%%%%%%%%%%%%%%%%%%%%%%%%%%%%%%%
%%%%%%%%%%%%%%%%%%%%%%%%%%%%%%%%%%%%%%%%%%%%%%%%%%%
%%%%%%%%%%%%%%%%%%%%%%%%%%%%%%%%%%%%%%%%%%%%%%%%%%%
%%%%%%%%%%%%%%%%%%%%%%%%%%%%%%%%%%%%%%%%%%%%%%%%%%%
%%%%%%%%%%%%%%%%%%%%%%%%%%%%%%%%%%%%%%%%%%%%%%%%%%%
%%%%%%%%%%%%%%%%%%%%%%%%%%%%%%%%%%%%%%%%%%%%%%%%%%%
\newpage

\appendix

\section{Alternative Notions of Regret}

In the online setting, one may consider notions of regret other than~\ref{eq:regret_alt2}.
One choice is
\[
\sum_{t=1}^{T}\ell(\inner{\bw_t,\Psi(\tilde{\bx}_t)},y_t)-\min_{\bw\in
  \Wcal}\sum_{t=1}^{T}\ell(\inner{\bw,\Psi(\tilde{\bx}_t)},y_t)
\]
but this is too easy, as it reduces to standard online learning with respect to
examples which happen to be noisy. Another kind of regret we may want to minimize is 
\begin{equation}\label{eq:regret_alt1}
    \sum_{t=1}^{T}\ell(\inner{\bw_t,\Psi(\tilde{\bx}_t)},y_t)
-
    \min_{\bw\in \Wcal}\ell(\inner{\bw_t,\Psi(\bx_t)},y_t)~.
\end{equation}
This is the kind of regret which is relevant for actually predicting the values
$y_t$ well based on the noisy instances. Unfortunately, in general this is too
much to hope for. To see why, assume we deal with a linear kernel (so that
$\Psi(\bx)=\bx)$, and assume $\ell(\bw,\bx,y)=(\inner{\bw,\bx}-y)^2$. Now,
suppose that the adversary picks some $\bw^{*}\neq \mathbf{0}$ in $\Wcal$,
which might be even known to the learner, and at each round $t$ provides the
example $(\bw^{*}/\norm{\bw^{*}},1)$. It is easy to verify that
\eqref{eq:regret_alt1} in this case equals
\[
\sum_{t=1}^{T}\left(\inner{\bw_t,\tilde{\bx}_t}-1\right)^2-\mathbf{0}~.
\]
Recall that the learner chooses $\bw_t$ before $\tilde{\bx}_t$ is
revealed. Therefore, if the noise which leads to $\tilde{\bx}_t$ has positive
variance, it will generally be impossible for the learner to choose $\bw_t$
such that $\inner{\bw_t,\tilde{\bx}_t}$ is arbitrarily close to $1$. Therefore,
the equation above cannot grow sub-linearly with $T$.

%%%%%%%%%%%%%%%%%%%%%%%%%%%%%%%%%%%%%%%%%%%%%%%%%% 
\section{Proof of \thmref{thm:main_Gaussian}}

The analysis in this subsection is similar to the one of
\subsecref{subsec:innprod}, focusing on Gaussian kernels. Namely, we assume
here that the kernel $k(\bx,\bx')$ is equal to
$e^{-\norm{\bx-\bx'}^2/\sigma^2}$ for some $\sigma^2>0$.

We start by constructing an explicit feature mapping $\Psi(\cdot)$ 
corresponding to the RKHS induced by our kernel. For any $\bx,\bx'$, we have 
that
\begin{align*}
&k(\bx,\bx')~=~ 
e^{-\norm{\bx-\bx'}^2/\sigma^2}~=~e^{-\norm{\bx}^2/\sigma^2}e^{-\norm{\bx'}^2/\sigma^2}e^{2\inner{\bx,\bx'}/\sigma^2}\\
&=~ 
e^{-\norm{\bx}^2/\sigma^2}e^{-\norm{\bx'}^2/\sigma^2}\left(\summ_{n=0}^{\infty}\frac{(2\inner{\bx,\bx'})^n}{\sigma^{2n}n!}\right)\\
&=~e^{-\norm{\bx}^2/\sigma^2}e^{-\norm{\bx'}^2/\sigma^2}\left(\summ_{n=0}^{\infty}\summ_{k_1=1}^{d}\cdots\summ_{k_n=1}^{d} 
\frac{(2/\sigma^2)^n}{n!}x_{k_1}\cdots x_{k_n}x'_{k_1}\cdots x'_{k_n}\right).
\end{align*}
This suggests the following feature representation: for any $\bx$, $\Psi(\bx)$ 
returns an infinite-dimensional vector, indexed by $n$ and $k_1,\ldots,k_n \in 
\{1,\ldots,d\}$, with the entry corresponding to $n,k_1,\ldots,k_n$ being 
$e^{-\norm{\bx}^2/\sigma^2}\frac{(2/\sigma^2)^n}{n!}x_{k_1}\ldots x_{k_n}$. 
The dot product between $\Psi(\bx)$ and $\Psi(\bx')$ is similar to a standard 
dot product between two vectors, and by the derivation above equals 
$k(\bx,\bx')$ as required.

The idea of deriving an unbiased estimate of $\Psi(\bx)$ is the following:
first, we sample $N_1,N_2$ independently according to
$\Pr(N_1=n_1)=\Pr(N_2=n_2)=(p-1)/p^{n+1}$. Then, we query the oracle for $\bx$ for
$2N_1+N_2$ times to get $\tilde{\bx}^{1},\ldots,\tilde{\bx}^{(2N_1+N_2)}$, and
formally define $\tilde{\Psi}(\bx)$ as
\begin{equation}\label{eq:tildepsi_explicit2}
\tilde{\Psi}(\bx)=\frac{(-1)^{N_1}p^{N_1+N_2+2}2^{N_2}}{N_1!N_2!\sigma^{2N_1+2N_2}(p-1)^2}
\left(\prod_{j=1}^{N_1}\inner{\tilde{\bx}^{(2j-1)},\tilde{\bx}^{(2j)}}\right) 
\!\!\left(\summ_{k_1,\ldots,k_{N_2}=1}^{d}\!\!\!\!\tilde{x}^{(2N_1+1)}_{k_1}
\!\!\cdots 
\tilde{x}^{(2N_1+N_2)}_{k_{N_2}}\be_{N_2,k_1,\ldots,k_{N_2}}\right)
\end{equation}
where $\be_{N_2,k_1,\ldots,k_{N_2}}$ represents the unit vector in the 
direction indexed by $N_2,k_1,\ldots,k_{N_2}$ as explained above.
Since the oracle calls are i.i.d., it is not hard to verify that the expectation of the expression above is
\begin{align*}
  &\left(\summ_{n_1=0}^{\infty}\frac{p-1}{p^{n_1+1}}
  \frac{(-1)^{n_1}p^{n_1+1}}{n_1!\sigma^{2n_1}(p-1)}(\inner{\bx,\bx})^{n_1}\right)
\!\!
  \left(\summ_{n_2=0}^{\infty}\frac{p-1}{p^{n_2+1}}\frac{p^{n_2+1}2^{n_2}}{n_2!\sigma^{2n_2}(p-1)}
    \summ_{k_1,\ldots,k_{n_2}=1}^{d}
\!\!\!
    x_{k_1}\cdots x_{k_{n_2}}\be_{n_2,k_1,\ldots,k_{n_2}}\right)\\
  &=~\left(\summ_{n_1=0}^{\infty}\frac{(-\norm{\bx}^2/\sigma^2)^{n_1}}{n_1!}\right)
  \left(\summ_{n_2=0}^{\infty}\frac{(2/\sigma^2)^{n_2}}{n_2!}
    \summ_{k_1,\ldots,k_{n_2}=1}^{d}x_{k_1}\cdots x_{k_{n_2}}\be_{n_2,k_1,\ldots,k_{n_2}}\right)\\
  &=~e^{-\norm{\bx}^2/\sigma^2}\left(\summ_{n_2=0}^{\infty}
    \summ_{k_1,\ldots,k_{n_2}=1}^{d}\frac{(2/\sigma^2)^{n_2}}{n_2!}x_{k_1}
    \cdots x_{k_{n_2}}\be_{n_2,k_1,\ldots,k_{n_2}}\right)
\end{align*}
which is exactly $\Psi(\bx)$ as defined above.

To actually store $\tilde{\Psi}(\bx)$ in memory, we simply keep and 
$\tilde{\bx}^{(1)},\ldots,\tilde{\bx}^{(2N_1+N_2)}$. The representation above is 
used implicitly when we calculate dot products between $\tilde{\Psi}(\bx)$ 
and other elements in the RKHS, via the subroutine \texttt{Prod}. 
We formalize the needed properties of $\tilde{\Psi}(\bx)$ in the following lemma.
\begin{lemma}\label{lem:tilde_unbiased2}
Assuming the construction of $\tilde{\Psi}(\bx)$ as in the discussion
above, it holds that $\E_t[\tilde{\Psi}(\bx)]=\Psi(\bx)$ for all
$\bx$. Moreover, if the noisy sample $\tilde{\bx}_t$ returned by the oracle
$A_t$ satisfies $\E[\norm{\tilde{\bx}_t}^2]\leq B_{\tilde{\bx}}$, then
\[
    \E\left[\norm{\tilde{\Psi}(\bx_t)}^2\right]
\leq
    \left(\frac{p}{p-1}\right)^2e^{(\sqrt{p}B_{\tilde{\bx}}+2p\sqrt{B_{\tilde{\bx}}})/\sigma^2}
\]
\end{lemma}
\begin{proof}
The first part of the lemma follows from the discussion above.
As to the second part, note that by~(\ref{eq:tildepsi_explicit2}), we have that
\begin{align*}
&
    \norm{\tilde{\Psi}(\bx_t)}^2
~=~
    \frac{p^{2N_1+2N_2+4}2^{2N_2}}{\bigl(N_1!N_2!\sigma^{2N_1+2N_2}(p-1)^2\bigr)^2}
    \left(\prod_{j=1}^{N_1}\bigl(\inner{\tilde{\bx}^{(2j-1)},\tilde{\bx}^{(2j)}}\bigr)^2\right)
  \left(\summ_{k_1,\ldots,k_{N_2}=1}^{d}\left(\tilde{x}^{(2N_1+1)}_{k_1}\ldots
  \tilde{x}^{(2N_1+N_2)}_{k_{N_2}}\right)^2\right)
\\ &=
  \frac{p^{2N_1+2N_2+4}2^{2N_2}}{\bigl(N_1!N_2!\sigma^{2N_1+2N_2}(p-1)^2\bigr)^2}
  \left(\prod_{j=1}^{N_1}\bigl(\inner{\tilde{\bx}^{(2j-1)},\tilde{\bx}^{(2j)}}\bigr)^2\right)
  \left(\prod_{j=1}^{N_2}\norm{\tilde{\bx}^{(N_1+j)}}^2\right)
\\ &\leq~
  \frac{p^{2N_1+2N_2+4}2^{2N_2}}{\bigl(N_1!N_2!\sigma^{2N_1+2N_2}(p-1)^2\bigr)^2}
  B_{\tilde{\bx}}^{2N_1} B_{\tilde{\bx}}^{N_2}~.
\end{align*}
The expectation of this expression over $N_1,N_2$ is equal to
\begin{align*}
&
    \left(\sum_{n_1=0}^{\infty}\frac{p-1}{p^{n_1+1}}
    \frac{p^{2n_1+2}}{(n_1!\sigma^{2n_1}(p-1))^2}
    B_{\tilde{\bx}}^{2n_1}\right)\left(\sum_{n_2=0}^{\infty}
    \frac{p-1}{p^{n_2+1}}\frac{p^{2n_2+2}2^{2n_2}}{(n_2!\sigma^{2n_2}(p-1))^2}
    B_{\tilde{\bx}}^{n_2}\right)
\\ &=
    \left(\frac{p}{p-1}\right)^2\left(\sum_{n_1=0}^{\infty}
    \frac{(pB_{\tilde{\bx}}^2)^{n_1}}{(n_1!\sigma^{2n_1})^2}\right)
\left(\sum_{n_2=0}^{\infty}\frac{(4p^2B_{\tilde{\bx}})^{n_2}}{(n_2!\sigma^{2n_2})^2}\right)\\
&= \left(\frac{p}{p-1}\right)^2\left(\sum_{n_1=0}^{\infty}\left(\frac{(\sqrt{p}B_{\tilde{\bx}}/\sigma^2)^{n_1}}{n_1!}\right)^2\right)\left(\sum_{n_2=0}^{\infty}\left(\frac{(2p\sqrt{B_{\tilde{\bx}}}/\sigma^2)^{n_2}}{n_2!}\right)^2\right)\\
&\leq \left(\frac{p}{p-1}\right)^2\left(\left(\sum_{n_1=0}^{\infty}\frac{(\sqrt{p}B_{\tilde{\bx}}/\sigma^2)^{n_1}}{n_1!}\right)
\left(\sum_{n_2=0}^{\infty}\frac{(2p\sqrt{B_{\tilde{\bx}}}/\sigma^2)^{n_2}}{n_2!}\right)\right)^2
~=~ \left(\frac{p}{p-1}\right)^2e^{(\sqrt{p}B_{\tilde{\bx}}+2p\sqrt{B_{\tilde{\bx}}})/\sigma^2}.
\end{align*}
\end{proof}
After the discussion above, the pseudocode for \texttt{Map\_Estimate} below 
should be self-explanatory. 

\begin{Subroutine}%[t]
 \caption{\texttt{Map\_Estimate}$(A_t,p)$} \label{sub:map2}
    \begin{Balgorithm}
       Sample $N_1$ according to $\Pr(N_1=n_1)=(p-1)/p^{n_1+1}$\\
       Sample $N_2$ according to $\Pr(N_2=n_2)=(p-1)/p^{n_2+1}$\\            
       Query $A_t$ for $2N_1+N_2$ times to get 
       $\tilde{\bx}^{(1)}_t,\ldots,\tilde{\bx}^{(2N_1+N_2)}_t$\\
       Return $\tilde{\bx}^{(1)}_t,\ldots,\tilde{\bx}^{(2N_1+N_2)}_t$ as 
       $\tilde{\Psi}(\bx_t)$.
    \end{Balgorithm}
 \end{Subroutine}

 We now turn to the subroutine $\texttt{Prod}$, which given two elements in the 
 RKHS, returns their dot product. This subroutine comes in two flavors: 
 either as a procedure defined over 
 $\tilde{\Psi}(\bx),\tilde{\Psi}(\bx')$ and returning 
 $\inner{\tilde{\Psi}(\bx),\tilde{\Psi}(\bx')}$ (\subref{sub:prod3}); or as a procedure 
 defined over $\tilde{\Psi}(\bx),\bx'$ (\subref{sub:prod4}, where the second element
 is an explicitly given vector) and returning  
 $\inner{\tilde{\Psi}(\bx),\Psi(\bx')}$. This second 
variant of \texttt{Prod} is needed when we wish to apply the hypothesis on a 
new (known) instance $\bx'$.
 
\begin{Subroutine}%[t]
 \caption{\texttt{Prod}$(\tilde{\Psi}(\bx),\tilde{\Psi}(\bx'))$} \label{sub:prod3}
    \begin{Balgorithm}
       Let $\tilde{\bx}^{(n)},\ldots,\tilde{\bx}^{(2n_1+n_2)}$ be the vectors comprising $\tilde{\Psi}(\bx)$\\
       Let $\tilde{\bx'}^{(1)},\ldots,\tilde{\bx'}^{(2n'_1+n'_2)}$ be the vectors comprising
       $\tilde{\Psi}(\bx')$\\ If $n'_2\neq n'_2$ return $0$, else return
       ${\displaystyle
       \frac{(-1)^{n_1+n_1'}p^{n_1+n_1'+2n_2+4}2^{2n_2}}{n_1!n'_1!(n_2!)^2\sigma^{2(n_1+n'_1+2n_2)}(p-1)^4}
       }$
       \\[2mm]
       \hspace{3cm}$\times\,\left(\prod_{j=1}^{n_1}\inner{\tilde{\bx}^{(2j-1)},\tilde{\bx}^{(2j)}}\right)
       \left(\prod_{j=1}^{n'_1}\inner{\tilde{\bx}'^{(2j-1)},\tilde{\bx}'^{(2j)}}\right)
       \left(\prod_{j=1}^{n_2}\inner{\tilde{\bx}^{(2n_1+j)},\tilde{\bx}'^{(2n'_1+j)}}\right)$
    \end{Balgorithm}
 \end{Subroutine}
 
The proof of the following lemma is a straightforward algebraic exercise, similar to the proof of \lemref{lem:prod_correct}. 
 \begin{lemma}\label{lem:prod_correct2}
    \texttt{Prod}$(\tilde{\Psi}(\bx),\tilde{\Psi}(\bx'))$ returns 
    $\inner{\tilde{\Psi}(\bx),\tilde{\Psi}(\bx')}$.
 \end{lemma}
  
The pseudocode for calculating the dot product 
$\inner{\tilde{\Psi}(\bx),\Psi(\bx')}$ (where $\bx'$ is known) is 
very similar, and the proof is essentially the same.
 
\begin{Subroutine}%[t]
 \caption{\texttt{Prod}$(\tilde{\Psi}(\bx),\bx')$} \label{sub:prod4}
    \begin{Balgorithm}
      Let $\bx^{(1)},\ldots,\bx^{(2n_1+n_2)}$ be the vectors comprising $\tilde{\Psi}(\bx)$\\
      Return\+\\
      ${\displaystyle
      \frac{(-1)^{n_1}p^{n_1+n_2+2}2^{2n_2}}{n_1!(n_2!)^2\sigma^{2(n_1+2n_2)}(p-1)^2}
      e^{-\norm{\bx'}^2/\sigma^2}\left(\prod_{j=1}^{n_1}\inner{\tilde{\bx}^{(2j-1)},\tilde{\bx}^{(2j)}}\right)
      \left(\prod_{j=1}^{n_2}\inner{\tilde{\bx}^{(2n_1+j)},\bx'}\right)
      }$.
    \end{Balgorithm}
\end{Subroutine}
 
We are now ready to prove \thmref{thm:main_Gaussian}. First, regarding
the expected number of queries, notice that to run \algoref{algo:kernel}, we
invoke \texttt{Map\_Estimate} and \texttt{Grad\_Length\_Estimate} once at round
$t$. \texttt{Map\_Estimate} uses a random number $2B_1+B_2$ of queries, where
$B_1,B_2$ are independent and distributed as
$\Pr(B_1=n)=\Pr(B_2=n)=(p-1)/p^{n+1}$. \texttt{Grad\_Length\_Estimate} invokes
\texttt{Map\_Estimate} a random number $C$ of times, where
$\Pr(C=n)=(p-1)/p^{n+1}$. The total number of queries is therefore
$\sum_{j=1}^{C+1}(2B_{j,1}+B_{j,2})$, where $B_{j,1},B_{j,2}$ are i.i.d. copies
of $B_{1},B_{2}$ respectively. The expected value of this expression, using a
standard result on the expected value of a sum of a random number of random
variables, is equal to $(1+\E[C])(2\E[B_{j,1}]+\E[B_{j,2}])$, or
$\left(1+\frac{1}{p-1}\right)\frac{3}{p-1}=\frac{3p}{(p-1)^2}$.

In terms of running time, the analysis is completely identical to the one
performed in the proof of \thmref{thm:main_inner_product}, and the expected
running time is the same up to constants. 

The regret bound in the theorem follows from \thmref{thm:kernel_main}, with the
expressions for constants following from \lemref{lem:grad_unbiased},
\lemref{lem:tilde_unbiased2}, and \lemref{lem:prod_correct2}.

%%%%%%%%%%%%%%%%%%%%%%%%%%%%%%%%%%%%%%%%%%%
\section{Proof of Examples \ref{ex:smooth_abs} and \ref{ex:smooth_hinge}}

Examples \ref{ex:smooth_abs} and \ref{ex:smooth_hinge} use the error function 
$\mathrm{Erf}(a)$ in order to construct smooth approximations of the hinge loss 
and the absolute loss (see \figref{fig:smoothed_losses}). The error function 
is useful for our purposes, since it is analytic in all of $\reals$, and 
smoothly interpolates between $-1$ for $a\ll 0$ and $1$ for $a\gg 0$. Thus, it 
can be used to approximate derivative of losses which are piecewise linear, 
such as the hinge loss $\ell(a)=\max\{0,1-a\}$ and the absolute loss 
$\ell(a)=|a|$.

To approximate the absolute loss, we use the antiderivative of  
$\mathrm{Erf}(sa)$. This function represents a smooth upper bound on the 
absolute loss, which becomes tighter as $s$ increases. It can be verified that 
the antiderivative (with the constant free parameter fixed so the function has 
the desired behavior) is
\[
\ell(a)= a~\mathrm{Erf}(sa)+\frac{e^{-s^2 a^2}}{\sigma\sqrt{\pi}}~. 
\]

While this loss function may seem to have slightly complex form, we note that 
our algorithm only needs to calculate the derivative of this loss function at 
various points (namely $\text{Erf}(sa)$ for various values of $a$), which can 
be easily done.

By a Taylor expansion of the error function, we have that 
\[
\ell'(a)=\frac{2}{\sqrt{\pi}}\sum_{n=0}^{\infty}\frac{(-1)^n (sa)^{2n+1}}{n!(2n+1)}~.
\]
Therefore, $\ell'_{+}(a)$ in this case is at most
\[
    \frac{2}{\sqrt{\pi}}\sum_{n=0}^{\infty}\frac{(sa)^{2n+1}}{n!(2n+1)}
\leq 
    \frac{2}{as\sqrt{\pi}}\sum_{n=0}^{\infty}\frac{(sa)^{2(n+1)}}{(n+1)!}
=
    \frac{2}{as\sqrt{\pi}}\left(e^{\sigma^2 a^2-1}\right)~.
\]
We now turn to deal with Example~\ref{ex:smooth_hinge}. This time, we use the 
antiderivative of $(\mathrm{Erf}(s(a-1))-1)/2$. This function smoothly 
interpolates between $-1$ for $a\ll -1$ and $0$ for $a\gg 0$. Therefore, its 
antiderivative with respect to $x$ represents a smooth upper bound on the 
hinge loss, which becomes tighter as $s$ increases. It can be verified that 
the antiderivative (with the constant free parameter fixed so the function has 
the desired behavior) is
\[
\ell(a)= 
\frac{(a-1)(\mathrm{Erf}(s(a-1))-1)}{2}+\frac{e^{-s^2(a-1)^2}}{2\sqrt{\pi}s}.
\]

By a Taylor expansion of the error function, we have that 
\[
\ell'(a)=-\frac{1}{2}+\frac{1}{\sqrt{\pi}}\sum_{n=0}^{\infty}\frac{(-1)^n 
(s(a-1))^{2n+1}}{n!(2n+1)}.
\]
Thus, $\ell'_{+}(a)$ in this case can be upper bounded by
\[
\frac{1}{2}+\frac{1}{\sqrt{\pi}}\sum_{n=0}^{\infty}\frac{(sa)^{2n+1}}{n!(2n+1)} 
\leq 
\frac{1}{2}+\frac{1}{as\sqrt{\pi}}\sum_{n=0}^{\infty}\frac{(sa)^{2(n+1)}}{(n+1)!} 
\leq \frac{1}{2}+\frac{1}{as\sqrt{\pi}}\left(e^{s^2a^2}-1\right).
\]

\section{Proof of \thmref{thm:kernel_main}} \label{app:ProofKernelMain}

Our algorithm corresponds to Zinkevich's algorithm \cite{Zinkevich03} in a
finite horizon setting, where we assume the sequence of examples is
$\tilde{g}_1\tilde{\Psi}(\bx_1),\ldots,\tilde{g}_T\tilde{\Psi}(\bx_T)$,
the cost function is linear, and the learning rate at round $t$ is
$\eta/\sqrt{T}$. By a straightforward adaptation of the standard regret bound
for that algorithm (see \cite{Zinkevich03}), we have that for any $\bw$ such
that $\norm{\bw}^2\leq B_{\bw}$,
\[
\sum_{t=1}^T 
\inner{\bw_t,\tilde{g}_t\tilde{\Psi}(\bx_t)}-\sum_{t=1}^T \inner{\bw,\tilde{g}_t\tilde{\Psi}(\bx_t)} ~\leq~ 
\frac{1}{2}\left(\frac{B_{\bw}}{\eta}+\frac{\eta}{T}\sum_{t=1}^{T}\norm{\tilde{g}_t\tilde{\Psi}(\bx_t)}^2\right)\sqrt{T}.
\]
We now take expectation of both sides in the
inequality above. The expectation of the right-hand side is simply
\[
    \E\left[\frac{1}{2}\left(\frac{B_{\bw}}{\eta}
+
    \frac{\eta}{T}\sum_{t=1}^{T}\E_t\bigl[\tilde{g}^2_t\bigr]\,
    \E_t\left[\norm{\tilde{\Psi}(\bx_t)}^2\right]\right)\sqrt{T}\right]
~\leq ~
    \frac{1}{2}\left(\frac{B_{\bw}}{\eta}+\eta
    B_{\tilde{g}}B_{\tilde{\Psi}}\right)\sqrt{T}.
\]
As to the left-hand side, note that
\[
\E\left[\sum_{t=1}^T\inner{\bw_t,\tilde{g}_t\tilde{\Psi}(\bx_t)}\right]
~=~\E\left[\sum_{t=1}^T\E_t\left[\inner{\bw_t,\tilde{g}_t\tilde{\Psi}(\bx_t)}\right]\right]
~=~\E\left[\sum_{t=1}^T\inner{\bw_t,y_t\ell'\bigl(y_t\inner{\bw_t,\Psi(\bx_t)}\bigr)\Psi(\bx_t)}\right].
\]
Also,
\[
\E\left[\sum_{t=1}^T \inner{\bw,\tilde{g}_t\tilde{\Psi}(\bx_t)}\right]
=
\sum_{t=1}^T \inner{\bw,\ell'\bigl(y_t\inner{\bw_t,\Psi(\bx_t)}\bigr)\Psi(\bx_t)}~.
\]
Plugging in these expectations and choosing $\eta=\sqrt{\frac{B_{\bw}}{B_{\tilde{g}}B_{\tilde{\Psi}}}}$, we get that for any $\bw$ such that $\norm{\bw}^2\leq B_{\bw}$,
\[
\E\left[\sum_{t=1}^T\inner{\bw_t,y_t\ell'\bigl(y_t\inner{\bw_t,\Psi(\bx_t)}\bigr)\Psi(\bx_t)}
-\sum_{t=1}^T \inner{\bw,\ell'\bigl(y_t\inner{\bw_t,\Psi(\bx_t)}\bigr)\Psi(\bx_t)}\right]
~\leq~ \sqrt{B_{\bw}B_{\tilde{g}}B_{\tilde{\Psi}}T}.
\]
To get the theorem, we note that by convexity of $\ell$, the left-hand side above
can be lower bounded by 
\[
\E\left[\sum_{t=1}^T\ell(y_t\inner{\bw_t,\Psi(\bx_t)})-\sum_{t=1}^T\ell(y_t\inner{\bw,\Psi(\bx_t)})\right].
\]

\section{Proof of Theorem~\ref{thm:impossibility2}}\label{app:impossibility2proof}
Fix a large enough $B_{\bw}\geq 1$ to be specified later.
Let $\bx = (1,0,\dots,0)$ and let $\Dcal$ to be the
uniform distribution on $\{3\bx,-\bx\}$.
To prove the result then we just need to show that
% \[
% \argmin{\bw\in
% \Wcal} \frac{\ell(\inner{\bw,3\bx},1)+\ell(\inner{\bw,-\bx})}{2}~~~\text{and}~~~ 
% \argmin{\bw\in \Wcal} \ell(\inner{\bw,\bx,1})
% \]
% are disjoint. In fact, since the components of $\bw$ that are orthogonal to $\bx$ 
% do not change the objective function, it suffices to look at $\bw$ which are a 
% linear scaling of $\bx$. Thus, it is enough to show that
\begin{equation}\label{eq:sets}
\argmin{w \,:\, |w|^2 \leq B_{\bw}} \ell(3w,1)+\ell(-w,1)
\qquad\text{and}\qquad
\argmin{w \,:\, |w|^2\leq B_{\bw}} \ell(w,1)
\end{equation}
are disjoint, for some appropriately chosen $B_{\bw}$.

First, we show that the first set above is a subset of $\{w:|w|^2\leq R\}$ for 
some fixed $R$ which does not depend on $B_{\bw}$. We do a case-by-case 
analysis, depending on how $\ell(\cdot,1)$ looks like.
\begin{enumerate}
\item $\ell(\cdot,1)$ monotonically increases in $\reals$. Impossible by assumption (2).
\item $\ell(\cdot,1)$ monotonically decreases in $\reals$. First, recall that
  since $\ell(\cdot,1)$ is convex, it is differentiable almost anywhere, and
  its derivative is monotonically increasing. Now, since $\ell(\cdot,1)$ is
  convex and bounded from below, $\ell'(w,1)$ must tend to $0$ as $w\rightarrow
  \infty$ (wherever $\ell(\cdot,1)$ is differentiable, which is almost
  everywhere by convexity). Moreover, by assumption (2), $\ell'(w,1)$ is upper
  bounded by a strictly negative constant for any $w<0$. As a result, the
  gradient of $\ell(3w,1)+\ell(-w,1)$, which equals $3\ell'(3w,1)-\ell'(-w,1)$,
  must be positive for large enough $w>0$, and negative for large enough $w<0$,
  so the minimizers of $\ell(3w,1)+\ell(-w,1)$ are in some bounded subset of
  $\reals$.
\item There is some $s\in \reals$ such that $\ell(\cdot,1)$ monotonically decreases
in $(-\infty,s)$ and
monotonically increases in $(s,\infty)$. If the function is constant in 
$(s,\infty)$ or in $(-\infty,s)$, we are back to one of the two previous 
cases. Otherwise, by convexity of $\ell(\cdot)$, we must have some $a,b$, 
$a\leq s\leq b$, such that $\ell(\cdot,1)$ is strictly decreasing at 
$(-\infty,a)$, and strictly increasing at $(b,\infty)$. In that case, it is 
not hard to see that $\ell(3w,1)+\ell(-w,1)$ must be strictly increasing for 
any $w>\max\{|a|,|b|\}$, and strictly decreasing for any $w<-\max\{|a|,|b|\}$. 
So again, the minimizers of $\ell(3w,1)+\ell(-w,1)$ are in some bounded subset 
of $\reals$.
\end{enumerate}
We are now ready to show that the two sets in~(\ref{eq:sets}) must be 
disjoint. Suppose we pick $B_{\bw}$ large enough so that the first set
in~(\ref{eq:sets}) is strictly inside $\{w \,:\, |w|^2\leq B_{\bw}\}$. Assume on the 
contrary that there is some $w,|w|^2<B_{\bw}$, which belongs to both sets
in~(\ref{eq:sets}). By assumption~(2) and the fact that $w$ minimizes 
$\ell(w,1)$, we may assume $w>0$. Therefore, $0\in 
\partial \ell(w,1)$ as well as $0\in 
\partial(\ell(3w,1)+\ell(-w,1))$, where $\partial f$ is the (closed and convex) subgradient set of 
a convex function $f$. By subgradient calculus, this means there is some 
$a/3\in 
\partial \ell(3w,1)$ and $b\in \partial \ell(-w,1)$ such that $a/3-b=0$. This implies 
that $\partial \ell(3w,1) \cap \partial \ell(-w,1)\neq \emptyset$. Now, 
suppose that $\max \partial \ell(-w,1)<0$. This would mean that $\min \partial 
\ell(3w,1)<0$. But then $\ell(\cdot,1)$ is strictly decreasing at $(w,3w)$, 
and in particular $\ell(w,1)>\ell(3w,1)$, contradicting the assumption that 
$w$ minimizes $\ell(\cdot,1)$. So we must have $\max 
\partial \ell(-w,1)\geq 0$. Moreover, $\min \partial \ell(-w,1)\leq 0$ (because $w$ minimizes $\ell(\cdot,1)$ and $-w<
w$). Since the subgradient set is closed and convex, it follows that $0\in 
\partial \ell(-w,1)$. Therefore, both $w$ and $-w$ minimize $\ell(\cdot,1)$. But this means 
that $\ell'(0)=0$, in contradiction to assumption (2).

\section{Proof of \thmref{thm:impossibility}}\label{app:impossibilityproof}
Let $\Dcal$ be a distribution which satisfies~(\ref{eq:impossibility}). The 
idea of the proof is that the learner cannot know if $\Dcal$ is the real 
distribution (on which regret is measured) or the distribution which includes
noise. Specifically, consider the following two adversary strategies: 
\begin{enumerate}
\item At each round, draw an example from $\Dcal$, and present it to the 
learner (with the label $1$) without adding noise.
\item At each round, pick the example $\E_{\Dcal}[\bx]$, add to it
zero-mean noise sampled from $Z-\E_{\Dcal}[\bx]$, where $Z$ is distributed
according to $\Dcal$, and present the noisy example (with the label $1$)
to the learner. 
\end{enumerate}
In both cases the examples presented to a learner appear to come from
the same distribution $\Dcal$. Hence, any learner observing one copy of
each example cannot know which of the two strategies is played by the
adversary. Since~(\ref{eq:impossibility}) implies that
the set of optimal learner strategies for each of the two adversary strategies
are disjoint, by picking an appropriate strategy the adversary can
force a constant regret.

To formalize this argument, fix any learning algorithm that observes one
copy of each example and let $\bw_1,\bw_2,\dots$ be the sequence of generated
predictors. Then it is sufficient to show that at least one of the following
two holds
\begin{align}
\label{eq:imp0}
    &\limsup_{T\to\infty}\max_{\bw\in \Wcal}
    \E\left[\frac{1}{T}\sum_{t=1}^T \ell(\inner{\bw_t,\bx_t},1)
    -  \ell\left(\inner{\bw,\bx_t},1\right)\right] > 0
\\
\label{eq:imp0-1}
    &\limsup_{T\to\infty}\frac{1}{T}\sum_{t=1}^T \ell(\inner{\bw_t,\E[\bx]},1)
    - \min_{\bw\in \Wcal}\ell\bigl(\inner{\bw,\E[\bx]},1\bigr) > 0
\qquad
    \mathrm{w.p.~1}
\end{align}
where in both cases the expectation is with respect to $\Dcal$
and ``w.p.~1'' refers to the randomness of the noise.
First note that~(\ref{eq:imp0}) is implied by
\begin{equation}
\label{eq:imp0-2}
    \limsup_{T\to\infty}
    \frac{1}{T}\sum_{t=1}^T \ell(\inner{\bw_t,\bx_t},1)
    - \min_{\bw\in \Wcal}\E\Bigl[\ell(\inner{\bw,\bx},1)\Bigr] > 0
\qquad
    \mathrm{w.p.~1.}
\end{equation}
Since $\Wcal$ is compact, $\Dcal$ is assumed to be
supported on a compact subset, and $\ell$ is convex and hence continuous, then
$\ell(\inner{\bw,\bx},1)$ is almost surely bounded. So by Azuma's inequality
\[
    \sum_{T=1}^\infty \Pr\left( \frac{1}{T}\sum_{t=1}^T
    \Bigl( \E_t\bigl[\ell(\inner{\bw_t,\bx},1)\bigr]
    - \ell(\inner{\bw_t,\bx_t},1) \Bigr)
\ge
    \epsilon \right)
<
    \infty
\qquad
    \forall \epsilon > 0~.
\]
where the expectation $\E_t[\,\cdot\,]$ is conditioned on the randomness in the
previous rounds.  Letting $\bar{\bw}_t = \frac{1}{t}\sum_{s=1}^t \bw_s$ (which
belongs to $\Wcal$ for all $t$ since it is a convex set), we have
\begin{equation*} %\label{eq:imp03}
    \frac{1}{T}\sum_{t=1}^T\ell(\inner{\bw_t,\bx_t},1)
\ge
    \frac{1}{T}\sum_{t=1}^T\E_t[\ell(\inner{\bw_t,\bx},1)]
\ge
    \E\Bigl[\ell\bigl(\inner{\bar{\bw}_T,\bx},1\bigr)\Bigr]
\end{equation*}
where the first inequality holds with
probability 1 as $T\to\infty$ by the Borel-Cantelli lemma, and the second one
holds for every $T$ because $\ell$ is convex.

Similarly, 
\begin{equation*} %\label{eq:imp06}
    \frac{1}{T}\sum_{t=1}^T\ell(\inner{\bw_t,\E[\bx]},1)
\ge
    \ell\bigl(\inner{\bar{\bw}_T,\E[\bx]},1\bigr)~.
\end{equation*}
Hence~(\ref{eq:imp0-1})--(\ref{eq:imp0-2}) are obtained if we show that
no single sequence of predictors $\bar{\bw}_1,\bar{\bw}_2,\ldots$ simultaneously
satisfies
\begin{equation}
\label{eq:impl}
    \limsup_{T\to\infty} F_1(\bar{\bw}_T) \le 0
\qquad\text{and}\qquad
    \limsup_{T\to\infty} F_2(\bar{\bw}_T) \le 0
\end{equation}
where
\[
    F_1(\bw_T) =
    \E\Bigl[\ell\bigl(\inner{\bar{\bw}_T,\bx},1\bigr)\Bigr]
    - \min_{\bw\in \Wcal} \E\bigl[\ell\left(\inner{\bw,\bx},1\right)\bigr]
\quad
    F_2(\bw_T) =
    \ell\bigl(\inner{\bar{\bw}_T,\E[\bx]},1\bigr)
    - \min_{\bw\in \Wcal}\ell\bigl(\inner{\bw,\E[\bx]},1\bigr)~.
\]
Suppose on the contrary that there was such a sequence.
Since $\bar{\bw}_T\in \Wcal$ for all $T$, and $\Wcal$ is compact, the sequence 
$\bar{\bw}_1,\bar{\bw_2},\ldots$ has at least a cluster point $\bar{\bw}\in \Wcal$.
Moreover, it is easy to verify that the functions $F_1$ and $F_2$ are continuous.
Indeed, $\ell(\inner{\cdot,\E[\bx]},1)$ is continuous by convexity of $\ell$ and
$\E[\ell(\inner{\cdot,\bx},1)]$ is continuous by the compactness assumptions. 
% and the max over two continuous functions is continuous).
Hence, any cluster point of $\bar{\bw}_1,\bar{\bw_2},\ldots$ is also a cluster
point of both $F_1$ and $F_2$. Since $F_1,F_2 \ge 0$ by construction, and we
are assuming that neither $F_1(\bar{\bw}) > 0$ nor $F_1(\bar{\bw}) > 0$ for any
cluster point $\bar{\bw}$, we must have $F_1(\bar{\bw}) = F_2(\bar{\bw})=0$.
But this means that $\bar{\bw}$ belongs to both sets appearing
in~(\ref{eq:impossibility}), in contradiction to the assumption they are
disjoint. Thus, no sequence of predictors satisfies~(\ref{eq:impl}), as
desired.

\end{document}